%% file: arXiv.tex
\documentclass{article} % For LaTeX2e

\RequirePackage[colorlinks,citecolor=blue,linkcolor=blue,urlcolor=blue,pagebackref]{hyperref}
\usepackage{arxiv,times}

\usepackage{algorithm}
\usepackage{algorithmic}

\usepackage{microtype}
\usepackage{graphicx}
\usepackage{subfigure}
\usepackage{booktabs} % for professional tables

\usepackage{hyperref}
\usepackage{amsmath}
\usepackage{amssymb}
\usepackage{mathtools}
\usepackage{amsthm}

\usepackage[capitalize,noabbrev]{cleveref}

\usepackage{graphicx}

\theoremstyle{plain}
\newtheorem{theorem}{Theorem}[section]
\newtheorem{proposition}[theorem]{Proposition}
\newtheorem{lemma}[theorem]{Lemma}

\theoremstyle{definition}
\newtheorem{definition}[theorem]{Definition}
\newtheorem{assumption}[theorem]{Assumption}
\theoremstyle{remark}

\newcommand{\pluseq}{\mathrel{{+}{=}}}
\usepackage{color}

\usepackage{url}

\input{math_commands.tex}

\input{def.tex}

\allowdisplaybreaks

\renewcommand{\eqref}[1]{(\ref{#1})}

\usepackage{url}

\title{Unified Perspective on Probability Divergence via Maximum Likelihood Density Ratio Estimation:
Bridging KL-Divergence and Integral Probability Metrics}

\renewcommand{\shorttitle}{Unified Perspective for Probability Divergence}

\newcommand*{\email}[1]{\texttt{#1}}
\newcommand*{\equalcontribution}[1][*]{\textsuperscript{*}}

\usepackage{authblk}

\author[1,2]{Masahiro Kato\thanks{\email{masahiro\_kato@cyberagent.co.jp}}$\ \ $}
\author[1,3]{Masaaki Imaizumi}
\author[4]{Kentaro Minami}

\affil[1]{The University of Tokyo}
\affil[2]{CyberAgent, Inc.}
\affil[3]{RIKEN Center for Advanced Intelligence Project}
\affil[4]{Preferred Networks, Inc.}

\begin{document}

\maketitle

\begin{abstract}
This paper provides a unified perspective for the Kullback-Leibler (KL)-divergence and the integral probability metrics (IPMs) from the perspective of maximum likelihood density-ratio estimation (DRE). 
Both the KL-divergence and the IPMs are widely used in various fields in applications such as generative modeling. However, a unified understanding of these concepts has still been unexplored.
In this paper, we show that the KL-divergence and the IPMs can be represented as maximal likelihoods differing only by sampling schemes, and use this result to derive a unified form of the IPMs and a relaxed estimation method.
To develop the estimation problem, 
we construct an unconstrained maximum likelihood estimator to perform DRE with a stratified sampling scheme.
We further propose a novel class of probability divergences, called the Density Ratio Metrics (DRMs), that interpolates the KL-divergence and the IPMs. 
In addition to these findings, we also introduce some applications of the DRMs, such as DRE and generative adversarial networks. In experiments, we validate the effectiveness of our proposed methods.
\end{abstract}

\section{Introduction}
The notion of divergence between probability measures plays an important role in statistics, machine learning, and information theory \citep{rachev1991probability}. Two of the widely used probability divergences are the Kullback-Leibler (KL) divergence \citep{Kullback51klDivergence} (an instance of $f$-divergence \citep{Ali1966,csiszar1967}), and the family of integral probability metrics \citep[IPMs,][]{Zolotarev1984,Muller1997}, including the Wasserstein distance \citep{Robert1975,Levina2001}, the maximum mean discrepancy  \citep[MMD,][]{borgwardt2006integrating,gretton2009}, and the Dudley metric \citep{Dudley2002}.

Density-ratio estimation (DRE) is a fundamental problem in statistics and has a long history \citep{silverman1978}. The obtained density ratios have a wide range of applications, such as regression under a covariate shift \citep{shimodaira2000improving,Reddi2015}, learning with noisy labels \citep{liu2014noisy,fang2020rethinking}, anomaly detection \citep{smola09a,Hido2011,Abe2019}, two-sample testing \citep{keziou2005,kanamori2010,sugiyama2011a}, causal inference \citep{Uehara2020}, change point detection \citep{kawahara2009}, and generative adversarial networks \citep[GANs,][]{Uehara2016}. 
In particular, density ratios appear in definitions of various probability divergences, such as $f$-divergences including the KL-divergence; hence DRE is also important for the application of these divergences.

Understanding the relation between the IPMs and KL-divergence 
%(or the $f$-divergence) 
using density ratios has been studied for a long time.
Inequality relations between these divergences and metrics have traditionally been investigated \citep{gibbs2002choosing,tsybakov2009introduction} along with their sample complexities \cite{Sriperumbudur2012,Liang2019}.
\citet{glaser2021kale} proposes a divergence that extends the KL-divergence and inherits properties of the MMD.
In the literature of GANs, \citet{Song2020} develops a method to generalize the $f$-GAN \citep{Goodfellow2014,Nowozin2016} and the Wasserstein-GAN \citep{Arjovsky2017}, where they are based on $f$-divergence and the Wasserstein distance, respectively.
\citet{Belavkin2018} and \citet{Ozair2019} consider the relationship in the context of mutual information. \citet{Agrawal2020} relates them in terms of their optimal lower bounds. 
Despite these advances, bridging the KL-divergence to the IPMs still remains unexplored.

In this paper, we elucidate a new connection between the KL-divergence and IPMs through  the development of new DRE schemes.
Specifically, we show that a solution of our scheme has two properties: (i) an optimal objective value coincides with the KL-divergence, and (ii) it has a form of the IPMs.
%Thus, from the viewpoint of DRE, we successfully bridge the KL-divergence and the IPMs.
Based on this result, we find that the IPMs with a certain function class can be written as a sum of the KL- and inverse KL-divergences.
%Also, we can perform DRE even if supports of distributions are not identical, by relaxing the DRE in the form of an IPM.

The DRE scheme that we develop for the above results is based on a nonparametric likelihood in a \textit{stratified} setting.
Stratified sampling is a framework for dealing with two samples, which has been studied mainly in the literature on causal inference \citep{Imbens1996,Wooldridge2001,Uehara2020}. In this setting, we obtain two groups of observations drawn from each population, and then we perform the maximum likelihood DRE, inspired by maximum likelihood density function estimation \citep{Good1971nature,Good1971,Montricher1975,tapia1978nonparametric,Scott1980,Eggermont1999}.
For estimating density functions, it is necessary to impose a constraint that the density function must integrate to one, which requires us to solve a constrained optimization problem. 
Because solving constrained problems is computationally challenging in general, we leverage a technique developed by \citet{Silverman1982}, which converts the constrained maximum likelihood problem to an equivalent unconstrained problem.
We extend these results to propose a maximum likelihood density \textit{ratio} estimation.
This scheme is different from Bregman divergence-based DRE summarized by \citet{Sugiyama:2012:DRE:2181148}.

As an application of our theoretical connection result, we develop a new class of probability divergences named a density ratio metric (DRM).
The DRMs possess several topological properties of both the KL-divergence and IPM, and serve as a valid probability divergence even for distributions that do not have common support. 
We also derive an upper bound on an error of the density ratio estimator.
In addition, we develop a DRM-based GAN as an IPM-based GAN method.
%generalization of D2GAN \citep{Nguyen2017} and also improves the generative ratio matching network \citep{Srivastava2020Generative}.

We summarize our findings and contributions as follows:
\begin{itemize}
\setlength{\parskip}{0cm}
\setlength{\itemsep}{0cm}
    \item Both the KL-divergence and the IPMs are written in the unified way as the maximum of our DRE scheme under the stratified sampling setting.
    \item The IPMs with a certain function class can be written as the sum of the KL and inverse KL-divergences. 
    \item We propose a novel probability divergence DRM, which bridges the density ratio, the KL-divergence and the IPMs;
\end{itemize}

The remainder of this paper is organized as follows. We first introduce the problem setting of DRE in Section~\ref{sec:problem} and show the maximum likelihood DREs in Section~\ref{sec:dre}. Then, in Section~\ref{sec:relate}, we discuss the relationship ship between the KL-divergence and IPMs. Based on the results, we define DRMs exhibiting some useful theoretical properties in Section~\ref{sec:drm}. Section~\ref{sec:regression_exp} presents the experimental results on DRE.

\section{Problem Setting of DRE}
\label{sec:problem}
We formulate the problem of DRE.
Let $\mathbb{P}$ and $\mathbb{Q}$ be two probability measures defined on a measurable space $\mathcal{W}$, which is a Borel subset of $\mathbb{R}^d$.
We assume that $\mathbb{P}$ and $\mathbb{Q}$ have the densities and denote them by $p^*$ and $q^*$, and also define their supports $\mathcal{W}_p, \mathcal{W}_q \in \mathcal{W}$ as $\mathcal{W}_p = \big\{x\in\mathcal{W}|  p^*(x)  > 0\big\}$ and  $\mathcal{W}_q = \big\{x\in\mathcal{W}| q^*(x) > 0\big\}$. 
We define their intersection $\mathcal{W}^* := \mathcal{W}_p \cap \mathcal{W}_q$.
Let $X\in\mathcal{W}$ and $Z\in\mathcal{W}$ be random variables following $\mathbb{P}$ and $\mathbb{Q}$. 
Suppose that we have two sets of observations $\mathcal{X} = \{X_i\}^n_{i=1}$ of size $n$ and $\mathcal{Z} = \{Z_j\}^m_{j=1}$ of size $m$, which are i.i.d.~samples from  $\mathbb{P}$ and $\mathbb{Q}$, respectively. 
%$\mathbb{E}_\mathcal{X}[\cdot]$ and $\mathbb{E}_\mathcal{X}[\cdot]$ is an expectation 

The goal of DRE is to estimate the density ratio between $p^*$ and $q^*$ or its inverse, which are defined as $r^*(x) = \frac{p^*(x)}{q^*(x)}$. Note that $r^*(x)$ (resp. $1/r^*(x)$) is not well-defined if $q^*(x) = 0$ (resp. $p^*(x) = 0$). 

\paragraph{Notation.} We denote by $\mathcal{P}(\mathcal{W})$
the set of probability measures defined on $\mathcal{W}$.
Let $\overline{R} > 1$ be a constant, which will be specified.
For an integration over $\mathcal{W}$, we simply denote it by $\int = \int_\mathcal{W}$. 
For a function $f:\mathcal{W}\to\mathbb{R}$ and a weight function $b:\mathcal{W}\to[1, \infty)$, we define a weighted norm $\|f\|_b = \sup_{x\in\mathcal{W}}\frac{|f(x)|}{b(x)}$.
Also, we define a function set $\mathcal{B}_b := \{f: \|f\|_b < \infty\}$. For a function $u:\mathcal{W}^*\to\mathbb{R}$, we denote the $L^2$ (pseudo-)norm over $\mathcal{W}^*$ with the probability measure $\mathbb{W}$ by $\|u\|_{L^2(\mathbb{W})} = (\int_{\mathcal{W}^*} u(x) \mathrm{d}\mathbb{W}(x))^{1/2}$ and the $L^\infty$ (pseudo-)norm by $\|u\|_{L^\infty(\mathbb{W})} = \sup_{x \in \mathcal{W}^*} |u(x)|$. Note that the expectation is defined only over $\mathcal{W}^*$, for which $r^*$ and $1/r^*$ are defined.

\section{Maximum Penalized Likelihood DRE (MPL-DRE)}
\label{sec:dre}
We consider a maximum penalized likelihood approach to DRE, as a preliminary step towards building a bridge between the KL-divergence and the IPMs.
%We mainly study a maximum penalized likelihood estimator (MPLE), which is known to avoid instability and overfitting to observations of naive likelihood maximization.
In this section, we briefly review classical nonparametric probability density estimation and its extension to DRE. % of the classical nonparametric probability density estimation to DRE. 
Next, we develop two formulations of DRE associated with different sampling schemes: the ordinary sampling and the stratified sampling.
In addition, we provide a convergence rate of the estimation error and discuss the choice of regularizers.
%we develop a novel likelihood framework for DRE, based on the stratified sampling scheme.
%we reconsider the framework of MPLE from the viewpoint of sampling schemes. By constructing the likelihood function based on stratified sampling schemes, we can obtain another likelihood for DRE. 

\subsection{Recap: MPLE of Probability Density Function}
Before discussing the maximum likelihood DRE, we review classical nonparametric maximum likelihood density estimation \citep{Good1971,Silverman1982}. Let $s: \mathcal{W} \to \mathbb{R}$ be a model of probability density $p$ and define the likelihood as $\prod^{n}_{i=1}s(X_i)$ and log-likelihood as $\sum^{n}_{i=1}\log s(X_i)$. We estimate $p(x)$ by maximizing the log-likelihood under the following constraint: $\int s(x) \mathrm{d}x = 1$. However, \citet{Good1971} finds that a naive application of maximum likelihood estimation would make the estimate the mean of a set of the Dirac functions at the $n$ observations, which is too rough as an estimate of the density function. To avoid this issue, \citet{Good1971} adds a \textit{roughness (smoothness) penalty} $\Psi(s) < \infty$ to the objective function of the log-likelihood to control the smoothness of the density function estimator. This framework is called maximum penalized likelihood estimation (MPLE). In nonparametric MPLE of the density, the objective is given as
\begin{align}
\label{eq:constrained_dens}
    &\ell(s) = \sum^{n}_{i=1}\log s(X_i) - \alpha \Psi(s),\\
    &\mathrm{s.t.}\ \int s(x) \mathrm{d}x = 1,\ \forall x\ s(x) \geq 0\nonumber,
\end{align}
where the positive number $\alpha$ is the smoothing parameter and $\Psi(s) < \infty$ is the roughness penalty, which is a functional. There are several candidates for the choice of the roughness penalty $\Psi(s)$, whose choice is discussed in Section~\ref{sec:roughness}.

\citet{Silverman1982} proposes an unconstrained formulation for nonparametric density estimation. Let $g \in \mathcal{G}$ be a model of the logarithmic density $\log s$, where $\mathcal{G}$ is a set of measurable functions. Then, it shows that the maximizer of
\begin{align*}
\sum^{n}_{i=1}g(X_i) - \int \exp(g(x)) \mathrm{d}x - \alpha \Psi(s)
\end{align*}
without constraint is identical with the maximizer of the constrained problem~\eqref{eq:constrained_dens}. We refer to this transformation as \textit{Silverman's trick}.

\begin{proposition}[Theorem~3.1 in \citet{Silverman1982}]
Suppose that $\Psi(s)$ only involves the derivative of $s(x)$ with regard to $x$. The function $\hat{g}$ in $\mathcal{G}$ minimizes $\sum^{n}_{i=1}g(X_i)$ over $g$ in $\mathcal{G}$ subject to $\int \exp(g) = 1$ if and only if $\hat{g}$ minimizes $\sum^{n}_{i=1}g(X_i) - \int \exp(g(x)) \mathrm{d}x$ over $g$ in $\mathcal{G}$.
\end{proposition}

Although a model of the logarithmic density is used in the original statement of \citet{Silverman1982}, we can remove this restriction as shown in \citet{Eggermont1999}. 

\subsection{MPL-DRE under the Ordinary Sampling}
We develop a novel MPLE framework for density \textit{ratios} named MPL-DRE, by extending the MPLE of the probability density. 
We first consider \textit{the ordinary sampling} setup, which considers a likelihood of a density ratio model using only one of $\mathcal{X}$ and $\mathcal{Z}$.
The stratified sampling, which utilizes both $\mathcal{X}$ and $\mathcal{Z}$, will be discussed in Section \ref{sec:mpl_dre_strat}.
%In this section, we consider the likelihood of the density ratio for each observation $\mathcal{X}$ and $\mathcal{Z}$. 
%We refer to this framework as  under the ordinary sampling, because a way of construction of likelihood follows conventional methods while we introduce another framework based on stratified sampling.

Let $r:\mathcal{W}\to(0,\infty)$ be a model of the density ratio $\frac{p^*(x)}{q^*(x)}$, which belongs to a function class $\mathcal{R}$ defined as follows.
\begin{definition}[proper function set]
\label{def:func_class}
A (measurable) function set $\mathcal{F}$ is \textit{proper}, if $\mathcal{F} \subset \mathcal{B}_b$ holds with a weight function $b:\mathcal{W}\to[1, \infty)$ as $b(x) = \min\{1, 1/q^*(x)\}$ for $x\in\mathcal{W}_q$ and $b(x) = 1$ for $x\notin\mathcal{W}_q$.
%We define a weight function $b:\mathcal{W}\to[1, \infty)$ as $b(x) = \min\{1, 1/q(x)\}$ for $x\in\mathcal{W}_q$ and $b(x) = 1$ for $x\notin\mathcal{W}_q$. Then, for the weight function $b$, a set $\mathcal{R} \subset \mathcal{B}_b$ contains measurable functions $r: \mathcal{W}\to (0, \infty)$ such that $r(x) = \overline{R}$ for all $x\notin\mathcal{W}_q$, and $r(x) = 1/\overline{R}$ for all $x\notin\mathcal{W}_p$. 
\end{definition}
With this definition, a proper function set $\mathcal{F}$ contains a function $r: \mathcal{W}\to (0, \infty)$ such that $r(x) = \overline{R}$ for all $x\notin\mathcal{W}_q$, and $r(x) = 1/\overline{R}$ for all $x\notin\mathcal{W}_p$. 

Here, by using the density ratio model $r$, a model of the density $p^*(x)$ (resp. $q^*(x)$) is written as $p_r(x) =r(x)q^*(x)$ (resp. $q_r = p^*(x)/r(x)$).
Using the models, we write a nonparametric likelihood for $r(x)$ as
%\begin{align*}
$\mathcal{L}_{\mathrm{ordinary}, p}(r; \mathcal{X}) = \prod^{n}_{i=1}p_r(X_i) = \prod^{n}_{i=1}r(X_i)q^*(X_i)$,
%\end{align*}
hence its log-likelihood is given as
\begin{align*}
\ell_{\mathrm{ordinary}, p}(r; \mathcal{X}) &= \sum^{n}_{i=1}\Big(\log r(X_i) + \log q^*(X_i)\Big).
\end{align*}
Note that $\log q(X_i)$ is irrelevant to the optimization.
We also define the following term for a constraint on $r$.
We recall that $\mathcal{W}^* = \mathcal{W}^*$.
\begin{align*}
    T_1(r) := \int_{\mathcal{W}^*} r(z)q^*(z)\mathrm{d}z + \int_{\mathcal{W}_p \cap\mathcal{W}^c_q}p^*(x)\mathrm{d}x.
\end{align*}
$T_1(r)=1$ guarantees that $r$ is a density ratio function from an aspect of the ordinary sampling scheme with $p^*$.

We update the objective of the MPL-DRE by introducing the roughness penalty $\Psi$:
\begin{align*}
&\max_{r\in\mathcal{R}} J_{\mathrm{ordinary}, p}(r; \mathcal{X}) - \alpha \Psi(r), \mathrm{~~s.t.~} T_1(r) = 1,\\
&\mathrm{where~~}J_{\mathrm{ordinary}, p}(r; \mathcal{X}) = \frac{1}{n}\sum^{n}_{i=1}\log r(X_i).
%&\mathrm{s.t.}\  \int_{\mathcal{W}^*} r(z)q^*(z)\mathrm{d}z + \int_{\mathcal{W}_p \cap\mathcal{W}^c_q}p^*(x)\mathrm{d}x= 1.
\end{align*}
In addition, inspired by Silverman's trick \citep{Silverman1982}, we consider the following unconstrained problem:
\begin{align*}
&\max_{r\in\mathcal{R}}\Bigg\{ J_{\mathrm{ordinary}, p}(r; \mathcal{X})  - \int_{\mathcal{W}^*} r(z)q^*(z)\mathrm{d}z - \alpha \Psi(r)\Bigg\}.
\end{align*}

To interpret the objective functions above, we study a problem of replacing the empirical summatinos of the objective functions with its expected value.
%We can show that the maximizer of the unconstrained problem is identical to that of the constrained problem when the objective is replaced with the expected value. 
We consider the following maximizers of the expected version of the objectives:
\begin{align*}
    &\tilde{r}_{\mathrm{ordinary}, p} := \argmax_{r \in \mathcal{R}: T_1(r)=1} \mathbb{E}_{\mathcal{X} }[J_{\mathrm{ordinary}, p}(r; \mathcal{X})] - \alpha \Psi(r),\\
    &r^\dagger_{\mathrm{ordinary}, p} := \argmax_{r \in \mathcal{R}} \mathbb{E}_{\mathcal{X} }[J_{\mathrm{ordinary}, p}(r; \mathcal{X})] - \int_{\mathcal{W}^*} r(z)q^*(z)\mathrm{d}z - \alpha \Psi(r),
\end{align*}
where $\mathbb{E}_{\mathcal{X}}$ denotes the expectation over $\mathcal{W}_p$ with respect to $\mathbb{P}$.
%Let $r^\dagger_{\mathrm{ordinary}, p}$ be the solution of $\max_{r\in\mathcal{R}}A(r)$, and $\tilde{r}_{\mathrm{ordinary}, p}$ be the solution of $\max_{r\in\mathcal{R}}A_0(r)$ under $T_1(r) = 1$. 
Then, we have the following theorem. The proof is inspired by \citet{Silverman1982} and shown in Appendix~\ref{appdx:thm:ordinary_silver}.

\begin{theorem}
\label{thm:ordinary_silver}
Suppose that the function class $\mathcal{R}$ follows Definition~\ref{def:func_class}, and $\Psi(r)$ only involves the derivative of $r(x)$ with regard to $x$. If $\mathcal{R}$ contains a function $r$ such that $r(x) = r^*(x)$ for all $x\in \mathcal{W}^*$, then $r^\dagger_{\mathrm{ordinary}, p} = \tilde{r}_{\mathrm{ordinary}, p}$.
\end{theorem}

Besides, the following theorem shows the analytical solution of $r^\dagger_{\mathrm{strat}}$. The proof is shown in Appendix~\ref{appdx:proof:opt_ordinary_p}. 
\begin{theorem}
\label{thm:opt_ordinary_p}
For $\mathcal{R}$ in Definition~\ref{def:func_class},
\begin{align*}
\tilde{r}_{\mathrm{ordinary}, p}(x) = \begin{cases}r^*(x) & \mathrm{if}\quad x \in \mathcal{W}^*\\
\overline{R} &\mathrm{if}\quad x \notin \mathcal{W}_p\\
\frac{1}{\overline{R}} &\mathrm{if}\quad x \notin \mathcal{W}_q\\
\end{cases}.
\end{align*}
\end{theorem}
Note that for $x \notin \mathcal{W}_p\cup \mathcal{W}_q$, Definition~\ref{def:func_class} gives the solution.

In estimation, by replacing the expectation in the unconstrained problem with the sample average and $\mathcal{R}$ with a hypothesis class $\mathcal{H}$, we solve the the following problem:
\begin{align*}
&\max_{r\in\mathcal{H}}\Bigg\{ J_{\mathrm{ordinary}, p}(r; \mathcal{X})  - \frac{1}{m}\sum^m_{j=1}r(Z_j) - \alpha \Psi(r)\Bigg\}.
\end{align*}
Similarly, we define the MPLE with unconstrained optimization problem of the reciprocal of the density ratio as 
\begin{align*}
&\max_{g\in\mathcal{G}}\Bigg\{ J_{\mathrm{ordinary}, q}(r; \mathcal{Z})  - \alpha \Psi(r)\Bigg\},
\end{align*}
where $J_{\mathrm{ordinary}, q}(r; \mathcal{Z}) = -\frac{1}{m}\sum^{m}_{j=1}\log r(Z_j)$.
As well as $\tilde{r}_{\mathrm{ordinary}, p}$ and Theorems~\ref{thm:opt_ordinary_p}, we denote the solution in expectation by $\tilde{r}_{\mathrm{ordinary}, q}$ and obtain the analytical solution. Then, we can confirm that $r^\dagger_{\mathrm{ordinary}, p}(x)=r^\dagger_{\mathrm{ordinary}(x), q}$.

Except for the penalties, the constrained optimization is identical to that of KL Importance Estimation Procedure \citep[KLIEP,][]{sugiyama2008}, and the unconstrained optimization is identical to that of \citet{Nguyen2008}. While their formulations are motivated by the minimization of the KL divergence or variational representations, our objectives are derived from the likelihoods (see Section~\ref{sec:kl_mpl}).
%Besides, we find imposing the constraint is unnecessary by Silverman's trick (see Section~\ref{sec:kl_mpl}).

\subsection{MPL-DRE under the Stratified Sampling}
\label{sec:mpl_dre_strat}
In the previous section, we defined the likelihood for each observation $\mathcal{X}$ and $\mathcal{Z}$ separately. 
Next, we define the likelihood of the density ratio using both $\mathcal{X}$ and $\mathcal{Z}$. Following terminology in statistics, we refer to this framework as MPL-DRE under the standard \textit{stratified sampling} \citep{Imbens1996,Wooldridge2001,Uehara2020}.

The likelihood function under the stratified sampling scheme is given as
%\begin{align*}
$\mathcal{L}_{\mathrm{strat}}(r; \mathcal{X}, \mathcal{Z}) = \prod^{n}_{i=1}p_r(X_i)\prod^{m}_{j=1}q_{r}(Z_j)$. 
Using the relations $p_r(x) = r(x)q^*(x)$ and $q_{r}(z) = p^*(z)/r(z)$, the log-likelihood function is given as
\begin{align*}
\ell_{\mathrm{strat}}(r; \mathcal{X}, \mathcal{Z})&=  \sum^n_{i=1}\Big( \log r(X_i) + \log q^*(X_i)\Big) + \sum^{m}_{j=1}\Big( - \log r(Z_j) + \log p^*(Z_j)\Big).
\end{align*}
Note that $\log q^*(X_i)$ and $\log p^*(Z_j)$ are irrelevant to the MPLE. 
We can further generalize the likelihood by considering a weighted likelihood \citep{Wooldridge2001}, which is defined as $\ell^\lambda_{\mathrm{strat}}(r; \mathcal{X}, \mathcal{Z})=  \lambda \frac{1}{n}\sum^n_{i=1} \log r(X_i)
+ (1-\lambda)\frac{1}{m}\sum^{m}_{j=1}\big( - \log r(Z_j) \big)$ with $\lambda \in [0,1]$. By choosing $\lambda$ appropriately, we can make the estimation more acculately. For example, when we consider parametric models, \citet{Wooldridge2001} implies that appropriate choice of $\lambda$ minimizes the asymptotic variance. 

We also define the following term:
\begin{align*}
    T_2(r) := \int_{\mathcal{W}^*} \frac{1}{r(x)}p^*(x)\mathrm{d}x  + \int_{\mathcal{W}^c_p \cap \mathcal{W}_q}q^*(z)\mathrm{d}z.
\end{align*}
A constraint $T_2(r) = 1$ normalizes $r$ from the perspective of $q^*$. Then, the MPLE under stratified sampling is given as
\begin{align*}
\max_{r\in\mathcal{R}}\ &J_{\mathrm{strat}}(r;\lambda, \mathcal{X}, \mathcal{Z})  - \alpha \Psi(r), \mathrm{~s.t.~}T_1(r)=T_2(r)=1,
%\mathrm{s.t.}\  &\int_{\mathcal{W}^*} r(z)q^*(z)\mathrm{d}z + \int_{\mathcal{W}_p \cap\mathcal{W}^c_q}p^*(x)\mathrm{d}x= 1\nonumber\\
%    &\int_{\mathcal{W}^*} \frac{1}{r(x)}p^*(x)\mathrm{d}x  + \int_{\mathcal{W}^c_p \cap \mathcal{W}_q}q^*(z)\mathrm{d}z= 1,\nonumber
\end{align*}
where $J_{\mathrm{strat}}(r;\lambda, \mathcal{X}, \mathcal{Z}) := \lambda J_{\mathrm{ordinary}, p}(r; \mathcal{X}) + (1-\lambda) J_{\mathrm{ordinary}, q}(r; \mathcal{Z})$ with $\lambda \in [0,1]$.
%\begin{align}
%\label{eq:mle}
%&J_{\mathrm{strat}}(r;\lambda, \mathcal{X}, \mathcal{Z})= \frac{\lambda}{n}\sum^{n}_{i=1}\log r(X_i) - \frac{1 - \lambda}{m}\sum^{m}_{j=1}\log r(Z_j)\nonumber\\
%&= \lambda J_{\mathrm{ordinary}, p}(r; \mathcal{X}) + (1-\lambda) J_{\mathrm{ordinary}, q}(r; \mathcal{Z}).
%\end{align}

Similar to the ordinary sampling, we study maximizers of an expected version of the objective functions.
\begin{align*}
    &\tilde{r}_{\mathrm{strat}} := \argmax_{r \in \mathcal{R}: T_1(r)=T_2(r)=1} \mathcal{K}(r) - \alpha \Psi(r),\\
    &r^\dagger_{\mathrm{strat}} := \argmax_{r \in \mathcal{R}} \mathcal{K}(r) - \alpha \Psi(r)  - (1 - \lambda) \int_{\mathcal{W}^*} \frac{1}{r(x)}p^*(x) \mathrm{d}x - \lambda \int_{\mathcal{W}^*} r(z) q^*(z) \mathrm{d}z,
\end{align*}
where $\mathcal{K}(r)$ is an expected log-likelihood defined as
\begin{align}
\label{eq:exp:ll}
    &\mathcal{K}(r) := \mathbb{E}_{\mathcal{X}, \mathcal{Z} }[J_{\mathrm{strat}}(r;\lambda, \mathcal{X}, \mathcal{Z})]\\
    &=\lambda \int \log r(x) p^*(x)\mathrm{d}x
 - (1-\lambda) \int \log r(x)q^*(x)\mathrm{d}x.\nonumber
\end{align}
We can relate $\tilde{r}_{\mathrm{strat}}$ with $r^\dagger_{\mathrm{strat}}$ as the following theorem.
\begin{theorem}
\label{thm:strat_silver2}
Under the same conditions in Theorem~\ref{thm:ordinary_silver},
$\tilde{r}_{\mathrm{strat}} = r^\dagger_{\mathrm{strat}} = \tilde{r}_{\mathrm{ordinary}, p} = \tilde{r}_{\mathrm{ordinary}, q}$.
\end{theorem}
The proof is shown in Appendix~\ref{appdx:proof:thm:strat_silver2}.

We define an estimator of MPL-DRE under the stratified sampling by replacing the expectation with the sample average and $\mathcal{R}$ with a hypothesis class $\mathcal{H}$, 
\begin{align}
\label{eq:estimator}
\hat{r} & = \argmax_{r\in\mathcal{H}}\ \Big\{\widehat{\mathcal{K}}(r)  - \alpha \Psi(r) \Big\},
\end{align}
where $\widehat{\mathcal{K}}(r) = J_{\mathrm{strat}}(r;\lambda, \mathcal{X}, \mathcal{Z})
- \frac{1-\lambda}{n}\sum^n_{i=1} \frac{1}{r(X_i)} - \frac{\lambda}{m}\sum^m_{j=1} r(Z_j)$. 
%\end{align*}
%The estimation error of $\hat{r}$ is analyzed in Section~\ref{sec:appendix:strong-convexity}.

\subsection{Estimation Error Bounds}
\label{sec:appendix:strong-convexity}
We derive an estimation error bound for $\hat{r}$ defined in \eqref{eq:estimator} on the \(\Ltwo\) norm. We provide a generalization error bound in terms of the Rademacher complexities of a hypothesis class and the following assumption. 

\begin{assumption}
There exists an empirical maximizer \(\hr \in \argmax_{\r \in \rClass} \widehat{\mathcal{K}}(r) \) and a population maximizer \(\rbest \in \argmax_{\r \in \rClass}  \mathbb{E}_{\mathcal{X},\mathcal{Z}}[\widetilde{\mathcal{K}}(r)] \).
\label{assumption:main:est-error-bound}
\end{assumption}
In Theorem~\ref{thm:est_error_bound}, for a multilayer perception with ReLU activation function (Definition~\ref{appdx:sparse-network-function-class}), we derive the convergence rate of the \(\Ltwo\) distance. The proof is shown in Appendix~\ref{appdx:l2norm}. 
\begin{theorem}[\(L^2\) Convergence rate]
  \label{thm:est_error_bound}
  Let \(\rClass\) be defined as in
  Definitions~\ref{def:func_class} and \ref{appdx:sparse-network-function-class} and assume \(\rstar \in \rClass\).
  Under Assumption~\ref{assumption:main:est-error-bound}, for some \(0 < \gamma < 2\), as $m, n \to \infty$,
  \begin{align*}
  &\max\left\{\lambda\|\hat{r}  - r^*\|_{\Ltwo(\mathbb{Q})}, (1-\lambda)\|1/\hat{r} - 1/r^*\|_{\Ltwo(\mathbb{P})}\right\}= \Orderp{ \min\{n, m\}^{-1/(2+\gamma)}}.
  \end{align*}
%where $C_{r^*}=\max\left\{\| r^{*} \|^2_{L^\infty(\mathbb{P})},\| 1/r^{*} \|^2_{L^\infty(\mathbb{Q})} \right\}$.
\end{theorem}

Thus, DRE under the DRMs becomes nearly a parametric rate when $\gamma$ is close to zero. \citep{KanamoriStatistical2012a,Liang2019}. In addition to the convergence guarantee, this result is useful in some applications, such as causal inference \citep{chernozhukov2016,Uehara2020}. To complement this result, we empirically investigate the estimator error using an artificially generated dataset with the known true density ratio in Section~\ref{sec:regression_exp}. 

\subsection{MPL-DRE with Exponential Density Ratio Models}
When focusing on exponential density ratio models $\exp(g)$ for $g\in\mathcal{G}$, we can rewrite the objective function of MPL-DRE under the ordinary sampling as follows:
\begin{align*}
&\max_{g\in\mathcal{G}}\Bigg\{ J^e_{\mathrm{ordinary}, p}(g; \mathcal{X}) - \frac{1}{m}\sum^m_{j=1}\exp(g(Z_j)) - \alpha \Psi(r)\Bigg\},\nonumber\\
&\mbox{~where~}J^e_{\mathrm{ordinary}, p}(g; \mathcal{X}) = \frac{1}{n}\sum^{n}_{i=1}g(X_i).
\end{align*}
Similarly, we define an objective function for estimating the inverse density ratio as 
\begin{align*}
&\max_{g\in\mathcal{G}}\Bigg\{ J^e_{\mathrm{ordinary}, q}(g; \mathcal{Z}) - \frac{1}{n}\sum^n_{i=1}\exp(-g(X_i))  - \alpha \Psi(g)\Bigg\},\nonumber\\
&\mbox{~where~}J^e_{\mathrm{ordinary}, q}(g; \mathcal{Z}) = \frac{1}{m}\sum^{m}_{j=1}g(Z_j).
\end{align*}

Then, the objective in MPL-DRE under the standard stratified sampling is given as
\begin{align*}
&\max_{g\in\mathcal{G}}\ \Bigg\{J^e_{\mathrm{strat}}(g;\lambda, \mathcal{X}, \mathcal{Z}) - \frac{1}{n}\sum^n_{i=1}\exp\big(g(X_i)\big) - \frac{1}{m}\sum^m_{j=1}\exp\big(-g(Z_j)\big)- \alpha \Psi(g)\Bigg\},
\end{align*}
where
\begin{align}
\label{eq:emp_strat_exp}
    &J^e_{\mathrm{strat}}(g;\lambda, \mathcal{X}, \mathcal{Z}) := \lambda J^e_{\mathrm{ordinary}, p}(r; \mathcal{X}) +  (1 - \lambda)J^e_{\mathrm{ordinary}, q}(r; \mathcal{Z}).
\end{align}

\subsection{On the Roughness Penalties}
\label{sec:roughness}
We have hitherto introduced the MPLE of DRE under the ordinary and stratified sampling scheme. To prevent the estimates from boiling down to Dirac functions spiking at $\mathcal{X}$ and $\mathcal{Z}$,  we discuss several choices for the roughness penalty. 
In DRE, the roughness penalty by \citet{Good1971nature,Good1971} for density function $f$ is $\Psi(f) = \int_\mathbb{R} \frac{( f'(x))^2}{f(x)}\mathrm{d}x = 4\int_\mathbb{R} ((f(x)^{1/2})')^2\mathrm{d}x$,
which may also be considered as a measure of the ease of detecting small shifts in $r$. \citet{Silverman1982} proposes using $\Psi(f) = \int_\mathbb{R} ((\log f(x))^{'''})^2\mathrm{d}x$, which is a measure of higher-order curvature in $\log f$, which is zero if and only if $f$ is a Gaussian density function.

For simplicity of notation, we omit the roughness penalty from the objective function in the following sections, since the roughness penalty can also be interpreted as a choice of function class $\mathcal{R}$ \citep{Silverman1982}.

\section{Relationships between the KL-divergence and the IPMs from the Density-ratio Perspective}
\label{sec:relate}
First, we formally define the KL divergence and the IPMs. The KL divergence is defined as
\begin{align*}
    &\mathrm{KL}\big(\mathbb{P}\parallel \mathbb{Q}\big) := \int_{\mathcal{W}_p}  p^*(x) \log\frac{p^*(x)}{q^*(x)}\mathrm{d}x= \int_{\mathcal{W}_p}  p^*(x) \log r^*(x)\mathrm{d}x.
\end{align*}

For $\mathcal{F} \subset \mathcal{B}_b$ on $\mathcal{X}$, the IPMs based on $\mathcal{F}$ and between $\mathbb{P}, \mathbb{Q} \in \mathcal{P}(\mathcal{W})$ is defined as:
\begin{align*}
&\mathrm{IPM}_{\mathcal{F}}\big(\mathbb{P}\parallel \mathbb{Q}\big):=\sup_{f\in\mathcal{F}}\left\{ \int f({x}) p^*(x)\mathrm{d}x - \int f({x})q^*(x)\mathrm{d}x \right\}.
\end{align*}

If for all $f \in \mathcal{F}$, $-f \in \mathcal{F}$, then $\mathrm{IPM}_\mathcal{F}$ forms a metric over $\mathcal{P}(\mathcal{W})$; we assume that this is always true for $\mathcal{T}$ in this paper to enable the removal of the absolute values. There is an obvious trade-off in the choice of $\mathcal{F}$ to fully characterize the $\mathrm{IPM}_{\mathcal{F}}$; that is, on one hand, the function class must be sufficiently rich that $\mathrm{IPM}_{\mathcal{F}}$ vanishes if and only if $\mathbb{P} = \mathbb{Q}$. On the other hand, the larger the function class $\mathcal{F}$, the more difficult it is to estimate $\mathrm{IPM}_{\mathcal{F}}$ \citep{Muandet2017}. Thus, $\mathcal{F}$ should be restrictive enough for the empirical estimate to converge rapidly \citep{Sriperumbudur2012}.

We give examples of $\mathcal{F}$.
If we set $\mathcal{F} = \left\{f: \mathcal{W}\mapsto \mathbb{R}: |f(x) - f(y)|\leq \|x - y\|, (x,y)\in \mathcal{W}^2\right\}$, the corresponding IPM becomes the Wasserstein distance \citep{villani2008optimal}. 
If $\mathcal{F}$ is the reproducing kernel Hilbert space, the IPM coincides with the MMD \cite{Muandet2017}.

%Given an RKHS $\mathcal{H}$ associated to a kernel $k$, and two probability distributions $\mathbb{P}$ and $\mathbb{Q}$, the MMD is defined as the RKHS norm of the difference of mean embiddings of $\mathbb{P}$ and $\mathbb{Q}$: $\mathrm{MMD}\left(\mathbb{P}\parallel\mathbb{Q}\right) = \int k(x, \cdot)p^*(x)\mathrm{d}x - \int k(x, \cdot)d\mathbb{Q}$. 

\subsection{The KL-Divergence and MPL-DRE}
\label{sec:kl_mpl}

In this section, we elucidate the relationship between KL-divergence and MPL-DRE. Suppose that $\mathcal{W}_p\subseteq \mathcal{W}_q$. Let us denote by $\mathcal{G}$ the set of continuous bounded functions from $\mathcal{W}$ to $\mathbb{R}$. Let us consider the dual of the KL divergence, defined as follows  \citep{Donsker1976,ambrosio2005gradient,Nguyen2008,nguyen2010,arbel2021generalized}:
\begin{align}
\label{eq:KL}
    \sup_{g\in \mathcal{G}}\left\{ 1 + \int_{\mathcal{W}_p} g(x) p^*(x)\mathrm{d}x - \int_{\mathcal{W}^*} \exp(g(x)) q^*(x)\mathrm{d}x \right\}.
\end{align}
By Silverman's trick, the maximizer $g^*$ satisfies $\int \exp(g) q^*(x)\mathrm{d}x = 1$. Therefore, 
\begin{align*}
    &\mbox{\eqref{eq:KL}}= \int_{\mathcal{W}_p} g^*(x)p^*(x)\mathrm{d}x.
\end{align*}
If $\mathcal{G}$ includes the true logarithm of the density ratio function from the dual of the KL divergence, it may be noted that the maximized expected log-likelihood is identical to the KL divergence. We summarize this result in the following lemma, which is derived from Theorems~\ref{thm:ordinary_silver} and \ref{thm:opt_ordinary_p}.

\begin{lemma}
For $\mathcal{G}$, suppose that $\mathcal{R}=\{\exp(g)| g\in\mathcal{G}\}$ be a proper function set. If $\mathcal{R}$ contains a function $r$ such that $r(x) = r^*(x)$ for all $x\in\mathcal{W}^*$ and $\mathcal{W}_p\subseteq\mathcal{W}_q$, then the maximum expected log-likelihood under the ordinary sampling over the exponential density ratio models,
\begin{align}
\label{eq:kl_opt}
&\max_{g\in\mathcal{G}}\int_{\mathcal{W}_p} g(x)  \mathrm{d}\mathbb{P}(x),\mathrm{~s.t.~} \int_{\mathcal{W}^*} \exp(g(z)) \mathrm{d}\mathbb{Q}(x) = 1,
\end{align}
matches the KL-divergence $\mathrm{KL}\big(\mathbb{P}\parallel \mathbb{Q}\big)$.
\end{lemma}

This formulation is also identical to that of \citet{nguyen2010}, which estimates the density ratio by solving \eqref{eq:KL}. This paper motivates the method from the perspective of the likelihood and finds that this formulation has a normalization effect by Silverman's trick. In fact, \citet{sugiyama2008} proposes KLIEP, which solves the constrained optimization problem \eqref{eq:kl_opt}, and \citet{Sugiyama:2012:DRE:2181148} refers to the objective function of \citet{nguyen2010} as unnormalized KL-divergence (UKL) because it does not have a normalization term. However, as explained above, the maximizer is normalized owing to Silverman's trick without considering the constrained problem as \citet{sugiyama2008}. 
\eqref{eq:KL} is also called KL Approximate Lower bound Estimator (KALE) \citep{arbel2021generalized,glaser2021kale}.

\subsection{The IPMs and MPL-DRE}
%As shown in \eqref{eq:exp:ll}, we can find that the maximum expected log likelihood of the MPL-DRE under the stratified sampling surprisingly corresponds to the IPMs when we restrict the density ratio model to be an exponential function, $\lambda$ to be $0.5$, and $\mathcal{F}$ to be a set of functions such that $T_1(r)=T_2(r)=1$. The empirical form is shown in \eqref{eq:emp_strat_exp}. This finding means that a certain IPM measure the distance between distributions by the maximum log likelihood of their density ratio. 

Remarkably, under the stratified sampling scheme, the maximum expected log-likelihood of the MPL-DRE coincides with the IPMs with a certain function class. In particular, we can see this through \eqref{eq:exp:ll} by (i) considering the exponential-type density ratio model, (ii) setting $\lambda$ to be $0.5$, and (iii) setting $\mathcal{F}$ to be a set of functions such that $T_1(r)=T_2(r)=1$. We can also obtain the empirical counterpart from \eqref{eq:emp_strat_exp}. This finding means that, in a certain situation, the maximum log-likelihood of the density ratio defines a proper distance between corresponding probability distributions. 

As mentioned in Section~\ref{sec:roughness}, imposing the roughness penalty corresponds to a restriction on the function class $\mathcal{R}$, giving rise to variants of the IPMs.

%For the domain without the common support, the density ratio model achieves the upper or lower bounds.

%$\frac{1}{n}\sum^{n}_{i=1}\exp(-f(X_i)) = \frac{1}{m}\sum^{m}_{j=1}\exp(f(Z_j)) = 1$.

\citet{Nguyen2017} and \citet{Zhao_Cong_Dai_Carin_2020} also propose a sum of KL and inverse KL divergences, but they do not discuss the relationship between the sum and the IPMs. In fact, the D2GAN proposed by \citet{Nguyen2017} can also be considered as a variant of IPM-based GANs.

\section{The Density Ratio Metrics (DRMs)}
\label{sec:drm}

This paper introduces the DRMs as an unified set of probability divergences, which bridges KL-divergence and a certain IPM via the density ratio. 
We define the DRMs based on the expected weighted log-likelihood of the density ratio under the stratified sampling as
%\footnotesize
\begin{align}
&\mathrm{DRM}^\lambda_{\mathcal{R}}(\mathbb{P}\parallel\mathbb{Q}) :=\label{eq:drm_obj}\\
&\sup_{r\in C(\mathcal{R})}\left\{ \lambda \int \log r(x) \mathrm{d}\mathbb{P}(x)
 - (1-\lambda) \int \log r(z)\mathrm{d}\mathbb{Q}(z) \right\},\notag 
\end{align}
%\normalsize
where $\mathcal{R}$ is a set of measurable functions defined in Definition~\ref{def:func_class}, and the set of functions $C(\mathcal{R})$ is defined as 
\begin{align*}
    &C(\mathcal{R}) = \left\{r \in \mathcal{R}: T_1(r) = T_2(r) = 1\right\}.
%    &\int_{\mathcal{W}^*} r(z) q^*(z)\mathrm{d}z + \int_{\mathcal{W}_p\cap \mathcal{W}^c_q} p^*(x)\mathrm{d}x = 1,\\
%    &\int_{\mathcal{W}^*} \frac{1}{r(x)} p^*(x)\mathrm{d}x + \int_{\mathcal{W}^c_p\cap \mathcal{W}_q} q^*(z)\mathrm{d}z = 1\Bigg\}.
\end{align*}

As well as the previous section, we omit the roughness penalty $\Psi(r)$ by interpreting it the choice of function class $\mathcal{R}$.
In DRM, the optimal $r$ in \eqref{eq:drm_obj} is the density ratio as shown in Theorem~\ref{thm:strat_silver2}. Besides, as a probability divergence, the following lemma holds. 

\begin{lemma} \label{lem:temp_drm_kl}
    %Suppose that $r^*$ and $1/r^*$ exists.
  For $\lambda \in [0,1]$, with sufficiently large $\overline{R}$,  $\mathrm{DRM}^\lambda_{\mathcal{R}}(\mathbb{P}\parallel\mathbb{Q}) = 0 \Leftrightarrow \mathbb{P} = \mathbb{Q}$ for any $\mathbb{P}, \mathbb{Q} \in \mathcal{P}(\mathcal{W})$.
\end{lemma}
The proof is shown in Appendix~\ref{appdx:lem:temp_drm_kl}.

We give an empirical approximator for $\mathrm{DRM}^\lambda_{\mathcal{R}}(\mathbb{P}\parallel\mathbb{Q})$.
Suppose we have empirical measures $\mathbb{P}_n = n^{-1} \sum_{i=1}^n \delta_{X_i}$ and $\mathbb{Q}_m = m^{-1} \sum_{j=1}^m \delta_{Z_j}$ with the Dirac measure $\delta_x$ at $x \in \mathcal{W}$.
As discussed in Section~\ref{sec:mpl_dre_strat}, we achieve the empirical approximation as % of $\mathrm{DRM}^\lambda_{\mathcal{R}}(\mathbb{P}\parallel\mathbb{Q})$ is given as 
\begin{align*}
&\widehat{\mathrm{DRM}}^\lambda_{\mathcal{R}, n, m}(\mathbb{P}_n\parallel\mathbb{Q}_m) := \sup_{r\in \mathcal{R}}\widehat{\mathcal{K}}(r).
\end{align*}

\subsection{From DRM to the KL-divergence and the IPMs}
We define a function class $\tilde{\mathcal{R}}$ as 
%\begin{align*}
    $\tilde{\mathcal{R}} = \left\{\exp(g(\cdot))\mid g \in \mathcal{G}\right\}$.
%\end{align*}
%where $\mathcal{G}$ is a class of measurable functions. 
Then, we obtain the following result.
\begin{theorem} \label{thm:drm_kl}
Suppose that under $\mathcal{G}$, $\tilde{\mathcal{R}}$ satisfies Definition~\ref{def:func_class}. Then,
\begin{align*}
    &\mathrm{DRM}^{1/2}_{\tilde{\mathcal{R}}}(\mathbb{P}\parallel\mathbb{Q}) = \frac{1}{2}\mathrm{IPM}_{\mathcal{G}}(\mathbb{P}\parallel\mathbb{Q}).
\end{align*}
Besides, suppose that $r^* \in \tilde{\mathcal{R}}$. If $\mathcal{W}_q \subseteq \mathcal{W}_p$, then $\mathrm{DRM}^{1}_{\tilde{\mathcal{R}}}(\mathbb{P}\parallel\mathbb{Q}) = \mathrm{KL}(\mathbb{P}\parallel\mathbb{Q})$. If $\mathcal{W}_p \subseteq \mathcal{W}_q$, then $\mathrm{DRM}^{0}_{\tilde{\mathcal{R}}}(\mathbb{P}\parallel\mathbb{Q}) = \mathrm{KL}(\mathbb{Q}\parallel \mathbb{P})$.
\end{theorem}
Thus, the DRMs bridge the IPMs and KL-divergence via the density ratio. We illustrate the concept in Figure~\ref{fig:concept}. 

In order to estimate the density ratio using finite samples, we need to control the roughness of the estimator appropriately, as explained in Section~\ref{sec:dre}. In the DRMs, the problem of roughness can be considered as a choice of function class $\mathcal{R}$, which corresponds to the property of the IPMs that the different function class yields different metrics, such as Wasserstein distance and MMD.

\begin{figure}[t]
  \begin{center}
    \includegraphics[width=80mm]{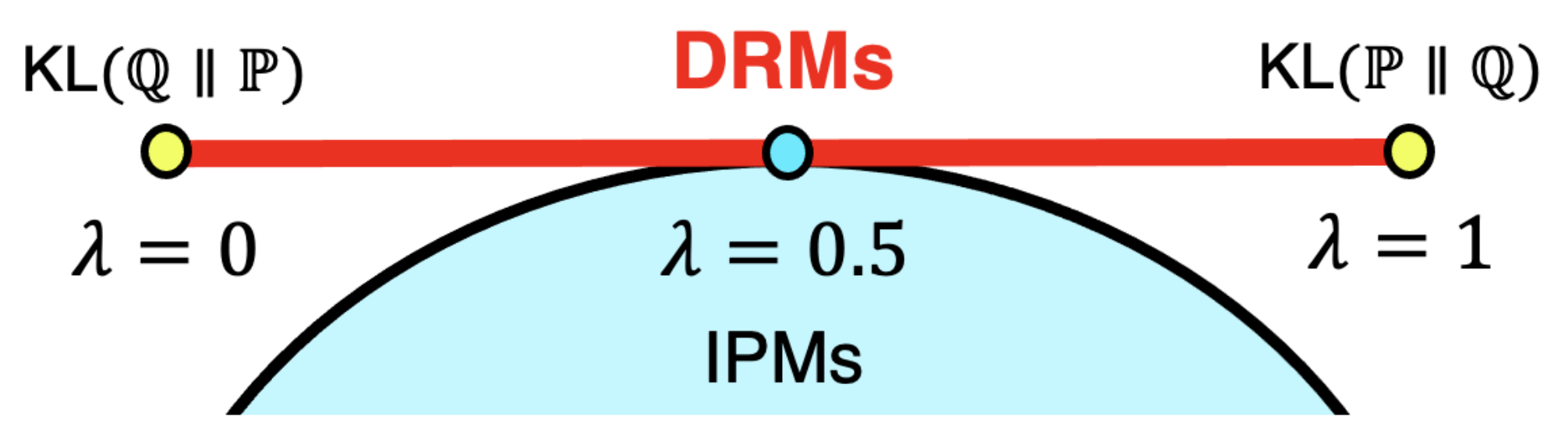}
  \end{center}
  \vspace{-0.3cm}
  \caption{Relationship between the KL-divergence, the IPMs, and DRMs. A family of the DRMs (red line) corresponds to the KL-divergence (yellow dots) at $\lambda \in \{0,1\}$, and is tangent to the set of the IPMs (blue region) at $\lambda = 0.5$.
  }
  \label{fig:concept}
  \vspace{-0.2cm}
\end{figure}

A map $d: \mathbb{P}\times \mathbb{Q} \mapsto d(\mathbb{P} \| \mathbb{Q}) \in [0, \infty]$ is called a probability semimetric if it possesses the following properties: (i) $d(\mathbb{P}\parallel \mathbb{Q}) = 0$ if and only if $\mathbb{P} = \mathbb{Q}$; (ii) $d(\mathbb{P}\parallel \mathbb{Q}) = d(\mathbb{Q}\parallel \mathbb{P})$; (iii) $d(\mathbb{P}\parallel \mathbb{Q}) \leq d(\mathbb{P}\parallel \mathbb{O}) + d(\mathbb{O}\parallel \mathbb{Q})$, where $\mathbb{O}$ is a probability measure. It is known that the IPMs are probability semimetric, and thus the following corollary holds.
\begin{theorem}
Suppose the same condition in Theorem~\ref{thm:drm_kl}. Then,
$\mathrm{DRM}^{1/2}_{\tilde{\mathcal{R}}}(\mathbb{P}\parallel\mathbb{Q})= \frac{1}{2}\mathrm{IPM}_{\mathcal{G}}(\mathbb{P}\parallel\mathbb{Q})$ is probability semimetric.
\end{theorem}

\begin{table*}[t]
\caption{Results of Section~\ref{sec:regression_exp}: means, medians, and stds of the squared error in DRE using synthetic datasets. The lowest mean and median (med) methods are highlighted in bold.}
\label{table:l2_error}
\vspace{-3mm}
\begin{center}
\scalebox{0.70}{
\begin{tabular}{|r|rrr|rrr|rrr|rrr|rrr|rrr|}
\hline
\multirow{2}{*}{dim} &    \multicolumn{3}{c|}{uLSIF}  & \multicolumn{3}{c|}{RuLSIF $(\alpha = 0.1)$}  & \multicolumn{3}{c|}{WD}   & \multicolumn{3}{c|}{DRM $(\lambda = 0.5)$}   & \multicolumn{3}{c|}{DRM $(\lambda = 0.1)$}  & \multicolumn{3}{c|}{DRM $(\lambda = 0.9)$} \\
 &      mean  &      med  &      std  &      mean  &      med  &      std &      mean  &      med  &      std  &   mean  &      med  &      std &   mean  &      med  &      std &     mean  &      med  &      std \\
\hline
2 &  11.216 &   9.997 &  4.547 &  11.655 &  10.481 &  4.566 &   7.962 &   4.536 &  13.838 &  1.781 &  1.107 &  2.025 &  \textbf{1.565} &  \textbf{0.977} &  1.762 &   3.039 &   2.194 &  2.959 \\
10 &  15.512 &  14.673 &  3.218 &  15.285 &  14.292 &  3.178 &  12.164 &  13.033 &   4.879 &  3.638 &  3.222 &  1.861 &  \textbf{2.901} &  \textbf{2.488} &  1.583 &   5.577 &   4.885 &  2.151 \\
50 &  15.578 &  15.045 &  3.326 &  15.586 &  15.055 &  3.326 &  14.307 &  13.922 &   3.614 &  6.654 &  6.032 &  2.340 &  \textbf{4.879} &  \textbf{4.273} &  1.856 &  10.117 &   9.451 &  2.728 \\
100 &  16.479 &  15.356 &  4.194 &  16.318 &  15.054 &  4.220 &   9.040 &   6.976 &   6.849 &  \textbf{7.540} &  \textbf{6.748} &  2.812 &  8.785 &  8.305 &  2.266 &  12.163 &  11.074 &  3.552 \\
\hline
\end{tabular}
}
\end{center}
\vspace{-3mm}
\end{table*}

\subsection{Topological Properties}

We present topological properties of the DRMs, that is, their relation to weak convergence of probability distributions. 
Such the property is important for generative models, such as GANs and adversarial VAEs.
Here, $\rightharpoonup$ denotes weak convergence of probability measures.
%weak continuity, and metrizing the weak convergence of probability distributions. 
%We recall that a functional $\mathcal{D}(\cdot \parallel \cdot)$ is a probability divergence if both $\mathcal{D}(\mathbb{P}\parallel\mathbb{Q}) \geq 0$ and $\mathcal{D}(\mathbb{P}\parallel\mathbb{Q}) = 0 \Leftrightarrow \mathbb{P} = \mathbb{Q}$ for any $\mathbb{P}, \mathbb{Q} \in \mathcal{P}(\Omega)$. 

\begin{theorem}
\label{thm:toplogical}
  Let $\mathbb{P}$ be a probability measure and $(\mathbb{P}_N)_{N\geq 0}$ be a sequence of probability measures. 
  Suppose $\mathcal{W}$ is bounded.
  Then, we have the followings:
  \begin{enumerate}
    \setlength{\parskip}{0cm}
\setlength{\itemsep}{0cm}
      \item (imply weak convergence) For $\lambda \in [0,1]$, $\lim_{N\to \infty} \mathrm{DRM}^\lambda_{\mathcal{R}}(\mathbb{P}_N\parallel \mathbb{P}) = 0\Rightarrow \mathbb{P}_N \rightharpoonup \mathbb{P}  $
      \item (metrize weak convergence) For $\lambda = 1/2$, $ \lim_{N\to \infty} \mathrm{DRM}^\lambda_{\mathcal{R}}(\mathbb{P}_N\parallel \mathbb{P}) = 0 \Leftrightarrow \mathbb{P}_N \rightharpoonup \mathbb{P}$.
  \end{enumerate}
\end{theorem}
The proof is provided in Appendix~\ref{appdx:thm:topological}.

\subsection{GAN based on the DRMs}
Given observations $\mathcal{X}$ from the density $p^*$, the goal of GANs is to learn a generator, which generates samples similar to $\mathcal{X}$. The generator is a parametric function $G_\theta: \mathbb{R}^{d'}\to \mathcal{W}$, where $\theta\in\Theta$ is the parameter, $\Theta$ is the parameter space, and $d'\ll d$. We denote the function class by $ \{G_\theta\}_{\theta\in\Theta}$. Each function $G_\theta$ is applied to a $d'$-dimensional random variable $\epsilon$, and for the generator, we define a family of densities $\mathcal{G} = \{q_\beta\}_{\beta\in\mathcal{B}}$. Let us denote $m\in\mathbb{N}$ i.i.d.~samples drawn from the density $q_\beta$ by $\mathcal{W}_q$. In contrast, the discriminator $D$ belongs to a family of Borel functions from $\mathcal{W}$ to $(0,1)$, denoted by $\mathcal{D}$. 

Probability divergence plays an important role in GANs, such as the Wasserstein GAN \citep{Arjovsky2017,Bousquet2017} and MMD GAN. Based on the stratified MPL-DRE, we also propose \textbf{S}tratified \textbf{L}ikelih\textbf{o}od based \textbf{GAN} (SLoGAN). The SLoGAN train the generator and discriminator by solving the following minimax game:
\begin{align*}
\min_{\theta\in\Theta}\ \mathrm{DRM}^\lambda_{\mathcal{R}}(\mathbb{P}\parallel\mathbb{Q}_\theta).
\end{align*}
\citet{Nguyen2017} proposes D2GAN with the following objective:
$J(g, r_1, r_2) = \alpha \mathbb{E}_{p}\big[\log r_1(X_i))\big] - \mathbb{E}_p \big[r_1(X_i)) \big]- \mathbb{E}_q\big[ r_2(X_i))\big] +  \beta \mathbb{E}_{p}\big[\log r(X_i))\big]$.
This objective can be regarded as a variant of the DRMs and the IPMs. In Proposition~1 in \citet{Nguyen2017}, the authors show that the optimal discriminator $r_1$ is $\alpha r^*$ and the optimal discriminator $r_2$ is $\beta / r^*$, if separate discriminators ($r_1 = 1/r_2$) and exponential density models are not used, the objective of D2GAN surprisingly can be reduced to a metric belonging to the IPMs under $\alpha=\beta = 1$. For instance, if we use 1-Lipschitz functions for the class of the discriminators, the objective becomes the Wasserstein distance with the normalization constraints.

\subsection{Related Work}
%The followings are some other related studies.

\textbf{DRE methods}.
\citet{sugiyama2011bregman} and \citet{KatoTeshima2021} focus on the Bregman divergence (BD) minimization framework \citep{bregman1967relaxation} to provide a general framework that unifies various DRE methods, such as moment matching \citep{huang2007,gretton2009}, probabilistic classification \citep{qin1998,Cheng2004}, density matching \citep{Nguyen2008,nguyen2010}, density-ratio fitting \citep{JMLR:Kanamori+etal:2009}, and learning from positive and unlabeled data \citep{kato2018learning}.

%\paragraph{KLIEP.}
More closely related to our work is the KLIEP, a framework of DRE by density matching under the KL divergence \citet{sugiyama2008},
%, called the KL importance estimation procedure (KLIEP). 
Although the original implementation by \citet{sugiyama2008} solves the constraint problem, we can transform the problem to an unconstrained problem by applying the method of \citet{Silverman1982}. The solution of KLIEP is equal to the solution of empirical UKL,  $\min_{r\in\mathcal{R}} -\frac{1}{n}\sum^n_{i=1} \log r(X_i) + \frac{1}{m}\sum^m_{j=1} r(Z_j)$. Although \citet{sugiyama2008,Sugiyama:2012:DRE:2181148} introduce the normalization constraint to this objective and omit $\frac{1}{m}\sum^m_{j=1} r(Z_j)$, we do not have to conduct the transformation because the solution of the unconstrained problem satisfies the constraint.

%\paragraph{Overfitting in DRE.}
The roughness problem is also closely related to the overfitting problem in DRE, called train-loss hacking \citep{KatoTeshima2021} and density-chasm problem \citep{Rhodes2020a}. \citet{Rhodes2020a}, \citet{Ansari2020iclr}, \citet{kumagai2021metalearning}, and \citet{Choi2021} mainly focus on the support of two densities in population. On the other hand, \citet{KatoTeshima2021} considers that it is caused by the finite samples. The roughness problem is more related to the train-loss hacking because, as well as \citet{Good1971}, the estimated density ratio becomes a set of Dirac delta functions even if there is a common support between two densities, which causes the overfitting problem. We introduce the correction for our objective in Appendix~\ref{appendix:est}.

\textbf{Relation to GANs}.
We discuss the generative ratio matching networks \citep{Srivastava2020Generative}, which uses density ratio as a discriminator and use MMD to train the generator. Although they estimate the density ratio by using uLSIF by \citet{JMLR:Kanamori+etal:2009}, separately from the training generator with MMD, we can also estimate the density ratio from MMD; that is, based on our findings, we can train the discriminator and generator by using the same objective.

%\paragraph{Regularizations in GANs}
A density ratio is closely related to a discriminator in GANs \citep{Train2017}. We can enforce the smoothness by using the findings of GANs. \citet{Chu2020Smoothness} categorizes the smoothness of the discriminator, 
%The authors refers to $f:\mathcal{X}\to\mathbb{R}$ as  $\alpha$-Lipschitz if $|f(x) - f(y)|\leq \alpha\|x - y\|_2,\quad (x,y)\in \mathcal{W}^2$ and $f:\mathcal{X}\to\mathbb{R}$ as $\beta_1$-Lipschitz if 
%$|\nabla f(x) - \nabla f(y)|\leq \beta_2\|x - y\|_2,\quad (x,y)\in \mathcal{W}^2$. 
and find that we can enforce Lipschitz continuity by using some constraints, such as spectral normalization \citep{miyato2018spectral}. % and smooth activation. % and spectral normalization.

\section{Experiments}
\label{sec:regression_exp}
We empirically investigate the $L^2$ error $\|r - r^*\|_{L^2(\mathbb{Q})}$ in the proposed DRE based on DRM. We compare our DRM-based DRE with the uLSIF \citep{JMLR:Kanamori+etal:2009} and RuLSIF \citep{yamada2011}. For DRM, we choose $\lambda$ from $0.5$, $0.9$, and $0.1$. The model is $3$-layer perceptron with a ReLU activation function, where the number of the nodes in the middle layer is $32$. We also apply the same spectral normalization~\citep{miyato2018spectral} to enforce Lipschitz continuity. Therefore, when setting $\lambda = 0.5$ in DRM, the metric becomes WD with normalization constraints owing to Lipschitz continuity. We also show the results when we do not use the normalization constraints (just maximizing the log likelihood of MPL-DRE under the stratified sampling), denoted by WD.  For uLSIF and RuLSIF, we use an open-source implementation\footnote{\url{https://github.com/hoxo-m/densratio_py}.}, which uses a linear-in-parameter model with the Gaussian kernel \citep{KanamoriStatistical2012a}. Let the dimensions of the domain be $d$, $\mathbb{P} = \mathcal{N}(\mu_p, I_d)$, and $\mathbb{Q} = \mathcal{N}(\mu_q, I_d)$, where $\mathcal{N}(\mu, \Sigma)$ denotes the multivariate normal distribution with mean $\mu$ and $\Sigma$, and let $\mu_p$ and $\mu_q$ be $d$-dimensional vectors $\mu_p = (0,0,\dots, 0)^\top$ and $\mu_q = (1,0,\dots, 0)^\top$, where and $I_d$ is a $d$-dimensional identity matrix. We fix the sample sizes at $n = m = 1,000$ and choose $d$ from $\{2, 10, 50, 100\}$. To measure the performance, we use the mean, median (med), and standard deviation (std) of the squared errors over $50$ trials. Note that in this setting, we know the true density ratio $r^*$. The results are shown in Table~\ref{table:l2_error}. The proposed DRM-based DRE methods estimate the density ratio more accurately than the other methods with a lower mean and median of the squared error. We also show additional experimental results with different parameters in Appendix~\ref{appdx:sec:add_regression_exp}. From the additional results, we can find that appropriate choices of $\lambda$ lowers the squared error. Besides, we show experimental results on distribution modeling in Appendix~\ref{appdx:generate}.

\section{Conclusion}
We have shown that the difference in sampling schemes in the construction of the likelihood of the density ratio yields the KL-divergence and the IPMs. Based on this finding, we have proposed a novel family of probability divergence, the DRMs, including the KL-divergence and the IPMs. In addition to DRE, the DRMs are useful in various applications; the present work has aimed to reach a deeper understanding through their connection to the density ratio.

\bibliography{arXiv.bbl}
\bibliographystyle{asa}

\clearpage

\onecolumn
\appendix

\section{Proof of Theorem~\ref{thm:ordinary_silver}}
\label{appdx:thm:ordinary_silver}
\begin{proof}
Preliminary, we define the following terms:
\begin{align*}
    &A(r) = -\int_{\mathcal{W}_p} \log ( r(x) ) p^*(x)\mathrm{d}x  + \int_{\mathcal{W}^*} r(z)q^*(z)\mathrm{d}z  + \alpha \Psi(r),\\
    &A_0(r) =  -\int_{\mathcal{W}_p} \log (r(x)) p^*(x)\mathrm{d}x - \alpha \Psi(r).
\end{align*}

Given $r$ in $\mathcal{R}$, we define $r^\diamond$ as
\begin{align*}
&\log \left(r^\diamond(x) \right)= -\log \left(r(x)\right) - \log \left(\int_{\mathcal{W}^*} r(z)q^*(z)\mathrm{d}z + \int_{\mathcal{W}_p \cap\mathcal{W}^c_q}p^*(x)\mathrm{d}x\right).
\end{align*}
Here, we have
\begin{align*}
&r^\diamond (x) = \frac{ r (x)}{\int_{\mathcal{W}^*} r(z)q^*(z)\mathrm{d}z + \int_{\mathcal{W}_p \cap\mathcal{W}^c_q}p^*(x)\mathrm{d}x}.
\end{align*}
Therefore,
\begin{align*}
&\int_{\mathcal{W}_p } r^\diamond (x) w^*(x)\mathrm{d}x = \int_{\mathcal{W}^*} r^\diamond (z) q^*(z)\mathrm{d}z + \int_{\mathcal{W}_p \cap\mathcal{W}^c_q}p^*(x)\mathrm{d}x = \frac{\int_{\mathcal{W}^*} r (z) q^*(z)\mathrm{d}z + \int_{\mathcal{W}_p \cap\mathcal{W}^c_q}p^*(x)\mathrm{d}x}{\int_{\mathcal{W}^*} r(z)q^*(z)\mathrm{d}z + \int_{\mathcal{W}_p \cap\mathcal{W}^c_q}p^*(x)\mathrm{d}x} = 1,
\end{align*}
where $w^*(x) = \mathbbm{1}[x\in\mathcal{W}^*](q^*(x) - p^*(x)/r^\diamond(x)) + p^*(x)/r^\diamond(x)$.  Besides, from the condition, $\Psi(r^\diamond) = \Psi(r)$. 

Using the equality, we obtain the following relation by elementary manipulations:
\begin{align*}
A\left(r^\diamond \right) &=- \int_{\mathcal{W}_p} \log \left(r^\diamond(x)\right) p^*(x)\mathrm{d}x+ \int_{\mathcal{W}^*} r^\diamond(z)q^*(z)\mathrm{d}z+ \int_{\mathcal{W}_p \cap\mathcal{W}^c_q}p^*(x)\mathrm{d}x + \Psi(r^\diamond)\\
&=- \int_{\mathcal{W}_p } \left( \log \left(r(x)\right) - \log \left(\int_{\mathcal{W}^*} r(z)q^*(z)\mathrm{d}z + \int_{\mathcal{W}_p \cap\mathcal{W}^c_q}p^*(x)\mathrm{d}x\right)\right) p^*(x)\mathrm{d}x + 1 + \Psi(r)\\
&=- \int_{\mathcal{W}_p } \log \left(r(x)\right)  p^*(x)\mathrm{d}x + \log \left(\int_{\mathcal{W}^*} r(z)q^*(z)\mathrm{d}z + \int_{\mathcal{W}_p \cap\mathcal{W}^c_q}p^*(x)\mathrm{d}x\right) + 1+ \Psi(r)\\
&=A(r) - \int_{\mathcal{W}^*}
r(z)q^*(z)\mathrm{d}z - \int_{\mathcal{W}_p \cap\mathcal{W}^c_q}p^*(x)\mathrm{d}x + \log\left(\int_{\mathcal{W}^*} r(z)q^*(z)\mathrm{d}z + \int_{\mathcal{W}_p \cap\mathcal{W}^c_q}p^*(x)\mathrm{d}x\right)  + 1\\
&=A(r) - T_1(r) + \log\left( T_1(r)\right)  + 1.
\end{align*}
Hence, we obtain $A(r^\diamond) \leq A(r)$.
Also, $A(r^\diamond) = A(r)$ holds only if $T_1(r)  = 1$, since $t - \log t \geq 1$ for all $t \geq 0$ and they are equal only if $t=1$. 

Therefore, $r$ minimizes $A(r)$ if and only if $r$ minimizes $A(r)$ subject to $T_1(r)  = 1$. Here, note that, subject to $T_1(r)  = 1$, the two objectives $A(r)$ and $A_0(r) + 1$ are identical. Thus, the proof is complete.
\end{proof}

\section{Proof of Theorem~\ref{thm:opt_ordinary_p}}
\label{appdx:proof:opt_ordinary_p}
\begin{proof}
We consider the minimization of
\begin{align*}
& -\int_{\mathcal{W}^*} \log ( r(x) ) p^*(x)\mathrm{d}x + \int_{\mathcal{W}^*} r(z)q^*(z)\mathrm{d}z + \int_{\mathcal{W}_p \cap\mathcal{W}^c_q}p^*(x)\mathrm{d}x,
\end{align*}
over all functions $r\in\mathcal{R}$. This problem can be reduced to the following point-wise minimization problem:
\begin{align}
&\min_{u \in (0,\infty)} -\log u q^*(x) + uq^*(x).
\end{align}

As we denote the solution by $u^* $, the first order condition of this minimization problem is given as
\begin{align*}
&-  \frac{1}{u^*}p^*(x) + q^*(x) = 0,
\end{align*}
and $u^* \in (0, \infty)$ holds.
Then, the solution given as $u^* = \frac{p^*(x)}{q^*(x)}\in(0,\infty)$.

Finally, we define $r^\dagger_{\mathrm{ordinary}, p}(x) :=\argmin_{u\in(0,\infty)} -\log u q^*(x) + uq^*(x)$ for $x\in\mathcal{W}^*$, $r^\dagger_{\mathrm{ordinary}, p}(x) = r^*(x)$. From Definition~\ref{def:func_class}, for $x\notin\mathcal{W}_p$, $r^\dagger_{\mathrm{ordinary}, p}(x) = 1/\overline{R}$, and for $x\notin\mathcal{W}_q$, $r^\dagger_{\mathrm{ordinary}, p}(x) = \overline{R}$. It can be confirmed that $r^\dagger_{\mathrm{ordinary}, p}(x)\in\mathcal{R}$ because $r^*$ is measurable and takes values in $(0, \infty)$. Therefore, the solution
of the original optimization problem is equal to $r^\dagger_{\mathrm{ordinary}, p}$ almost everywhere.

By the same procedure of the proof on $r^\dagger_{\mathrm{ordinary}, p}$, we can obtain the $r^\dagger_{\mathrm{ordinary}, q}$, which is equal to $r^\dagger_{\mathrm{ordinary}, p}$.
\end{proof}

\section{Proof of Theorem~\ref{thm:strat_silver2}}
\label{appdx:proof:thm:strat_silver2}
\begin{proof}
Preliminary, we define some supportive notations:
\begin{align*}
    &\mathcal{K}(r) = \lambda \int \log r(x) \mathrm{d}\mathbb{P}(x)
 - (1-\lambda) \int \log r(x)\mathrm{d}\mathbb{Q}(x).
\end{align*}

%Let $\tilde{r}_{\mathrm{strat}}$ be the maximizer of
%\begin{align*}
 %   \mathcal{K}(r) = \alpha \Psi(r) + \int_{\mathcal{W}_p }\lambda \log (r(x)) p^*(x)\mathrm{d}x -\int_{\mathcal{W}_q} (1 - \lambda) \log (r(z))  q^*(z)\mathrm{d}z,
%\end{align*}
%under the constraint in \eqref{eq:const_strat}, and $r^\dagger_{\mathrm{strat}}$ be the maximizer of
%\begin{align*}
%\widetilde{\mathcal{K}}(r) =  \mathcal{K}(r)-  (1 - \lambda) \int_{\mathcal{W}^*} \frac{1}{r(x)}p^*(x) \mathrm{d}x&\\
%- \lambda \int_{\mathcal{W}^*} r(z) q^*(z) \mathrm{d}z&.
%\end{align*}

We consider the KKT condition for functionals. Let us consider minimizing $-\mathcal{K}(r)$ for $r\in\mathcal{R}$, where $\mathcal{R}$ is defined in Definition~\ref{def:func_class}, satisfying $T_1(r) = T_2(r) = 1$.
%\begin{align*}
%    &\int_{\mathcal{W}^*} r(z)q^*(z)\mathrm{d}z + \int_{\mathcal{W}_p \cap\mathcal{W}^c_q}p^*(x)\mathrm{d}x= 1, \mbox{~~and~}\\
%    &\int_{\mathcal{W}^*} \frac{1}{r(x)}p^*(x)\mathrm{d}x  + \int_{\mathcal{W}^c_p \cap \mathcal{W}_q}q^*(z)\mathrm{d}z= 1.
%\end{align*}
We define a set of constrained functions as
\begin{align*}
    &\mathcal{T}_1 := \left\{r: T_1(r)= 1 \right\}, \mbox{~~and~~}\mathcal{T}_2 := \left\{r: T_2(r)= 1 \right\}.
\end{align*}
We consider the following inequality:
\begin{align}
    \min_{r\in\mathcal{R} \cap \mathcal{T}_1 \cap \mathcal{T}_2 }-\mathcal{K}(r)
%    &\min_{\stackrel{r\in\mathcal{R}}{\substack{\mathrm{s.t.}\ \int_{\mathcal{W}^*} r(z)q^*(z)\mathrm{d}z + \int_{\mathcal{W}_p \cap\mathcal{W}^c_q}p^*(x)\mathrm{d}x= 1,\\ \ \ \ \ \ \int_{\mathcal{W}^*} \frac{1}{r(x)}p^*(x)\mathrm{d}x  + \int_{\mathcal{W}^c_p \cap \mathcal{W}_q}q^*(z)\mathrm{d}z= 1}}}-\mathcal{K}(r)\\
    &\geq \min_{r\in\mathcal{R} \cap \mathcal{T}_1 \cap \mathcal{T}_2 }-\lambda \int \log r(x) \mathrm{d}\mathbb{P}(x) + \min_{r\in\mathcal{R} \cap \mathcal{T}_1 \cap \mathcal{T}_2 } (1-\lambda) \int \log r(x)\mathrm{d}\mathbb{Q}(x)\nonumber\\
    \label{eq:main}
    &\geq \min_{r\in\mathcal{R} \cap \mathcal{T}_1 }-\lambda \int \log r(x) \mathrm{d}\mathbb{P}(x) + \min_{r\in\mathcal{R} \cap \mathcal{T}_2 } (1-\lambda) \int \log r(x)\mathrm{d}\mathbb{Q}(x)
\end{align}
Then, we consider the solutions of
\begin{align*}
\min_{r\in\mathcal{R} \cap \mathcal{T}_1 }-\int \log r(x) \mathrm{d}\mathbb{P}(x),
\end{align*}
and 
\begin{align*}
\min_{r\in\mathcal{R} \cap \mathcal{T}_2 } - \int \log \frac{1}{r(x)}\mathrm{d}\mathbb{Q}(x),
\end{align*}
where recall that we denote the solutions by $r^\dagger_{\mathrm{ordinary}, p}(x)$ and $r^\dagger_{\mathrm{ordinary}, q}(x)$.

First, Theorem~\ref{thm:opt_ordinary_p} shows the solution $r^\dagger_{\mathrm{ordinary}, p}$. This result can be applied to $r^\dagger_{\mathrm{ordinary}, q}$. Then, we can confirm that $r^\dagger_{\mathrm{ordinary}, p}(x)=r^\dagger_{\mathrm{ordinary}(x), q} = r^\dagger$. By definition, $r^\dagger$ satisfies the constraints in \eqref{eq:main}. Therefore, $r^\dagger_{\mathrm{ordinary}, p}(x)=r^\dagger_{\mathrm{ordinary}(x), q} = r^\dagger$ is also the solution of \eqref{eq:main}. 
\end{proof}

\section{Proof of Theorem~\ref{thm:est_error_bound}}
\label{appdx:l2norm}
We consider relating the \(L^2\) error bound to the DRM generalization error bound in the following lemma.
\begin{lemma}[\(\Ltwo\) distance bound]
\label{appdx:sec:appendix:strong-convexity}
Let \(\mathcal{H} := \{\r: \InSpace \to (0, \infty)  | \int_{\mathcal{W}^*} |r(x)|^2 \mathrm{d}x < \infty\}\) and assume \(\rstar \in \mathcal{H}\).
If \(\inf_{t \in (0, \infty) } br''(t) > 0\), then there exists \(\mu > 0\) such that for all \(r \in \mathcal{H}\),
\begin{equation*}\begin{split}
&\lambda\|r  - r^*\|_{\Ltwo(\mathbb{Q})}^2/ \| r^{*} \|^2_{L^2(\mathbb{P})} + (1-\lambda)\|1/r - 1/r^*\|_{\Ltwo(\mathbb{P})}^2/\| 1/r^{*} \|^2_{L^2(\mathbb{Q})}\\
&= -  2\left(\widetilde{\mathcal{K}}(\r) - \widetilde{\mathcal{K}}(\rstar)\right)  -  o(\lambda \|r  - r^*\|_{\Ltwo(\mathbb{Q})}^2) - o((1-\lambda)\|1/r - 1/r^*\|_{\Ltwo(\mathbb{P})}^2).
\end{split}\end{equation*}
\end{lemma}
\begin{proof}
We define the additional notation
\begin{align*}
    \widetilde{\mathcal{K}}(r) =  \mathcal{K}(r)-  (1 - \lambda) \int_{\mathcal{W}^*} \frac{1}{r(x)}p^*(x) \mathrm{d}x - \lambda \int_{\mathcal{W}^*} r(z) q^*(z) \mathrm{d}z&.
\end{align*}

Since \(\mu := \inf_{t \in (0, \infty)} \log''(t) > 0\), the function \(\br\) is \(\mu\)-strongly convex.
By the definition of strong convexity,
\begin{equation*}\begin{aligned}
&-\left(\widetilde{\mathcal{K}}(\r) - \widetilde{\mathcal{K}}(\rstar)\right)\\
&= -\lambda \int_{\mathcal{W}^*} \log r(x) \mathrm{d}\mathbb{P}(x) + (1- \lambda) \int_{\mathcal{W}^*} \log r(z) \mathrm{d}\mathbb{Q}(z) + ( 1 - \lambda )\int_{\mathcal{W}^*} \frac{1}{r(x)} \mathrm{d}\mathbb{P}(x)  +  \lambda\int_{\mathcal{W}^*} r(z) \mathrm{d}\mathbb{Q}(z)  - 2\\
&\ \ \ \ + \lambda \int_{\mathcal{W}^*} \log r^*(x) \mathrm{d}\mathbb{P}(x) - (1- \lambda) \int_{\mathcal{W}^*} \log r^*(z) \mathrm{d}\mathbb{Q}(z) - ( 1 - \lambda )\int_{\mathcal{W}^*} \frac{1}{r^*(x)} \mathrm{d}\mathbb{P}(x)  -  \lambda\int_{\mathcal{W}^*} r^*(z) \mathrm{d}\mathbb{Q}(z)  + 2.
\end{aligned}\end{equation*}
Here, we have
\begin{align*}
& - \int_{\mathcal{W}^*} \log r(x) \mathrm{d}\mathbb{P}(x)  +  \int_{\mathcal{W}^*} r(z) \mathrm{d}\mathbb{Q}(z) + \int_{\mathcal{W}^*} \log r^*(x) \mathrm{d}\mathbb{P}(x)  -  \int_{\mathcal{W}^*} r^*(z) \mathrm{d}\mathbb{Q}(z)\\
&= - \int_{\mathcal{W}^*} r^*(z) \log r(z) \mathrm{d}\mathbb{Q}(x)  +  \int_{\mathcal{W}^*} r(z) \mathrm{d}\mathbb{Q}(z) + \int_{\mathcal{W}^*} r^*(z)\log r^*(z) \mathrm{d}\mathbb{Q}(z)  -  \int_{\mathcal{W}^*} r^*(z) \mathrm{d}\mathbb{Q}(z)\\
&= \int_{\mathcal{W}^*} \left\{- r^*(z) \log r(z)  +  r(z)  + r^*(z)\log r^*(z)  -  r^*(z)\right\} \mathrm{d}\mathbb{Q}(z)\\
&= \int_{\mathcal{W}^*} \left\{- r^*(z) \log \frac{r(z)}{r^*(z)}  +  r(z)  -  r^*(z)\right\} \mathrm{d}\mathbb{Q}(z)\\
&= \int_{\mathcal{W}^*} \left\{- r^*(z) \left( \frac{r(z)}{r^*(z)} - 1 -\frac{1}{2}\left(\frac{r(z)}{r^*(z)} - 1\right)^2 + \cdots \right)  +  r(z)  -  r^*(z)\right\} \mathrm{d}\mathbb{Q}(z)\\
&=\frac{1}{2}\int_{\mathcal{W}^*} \left(\frac{r(x)}{r^*(x)} - 1\right)^2 \mathrm{d}\mathbb{P}(x) + o\left(\int_{\mathcal{W}^*} \left(\frac{r(x)}{r^*(x)} - 1\right)^2 \mathrm{d}\mathbb{P}(x)\right).
\end{align*}
Similarly,
\begin{align*}
&\int_{\mathcal{W}^*} \log r(z) \mathrm{d}\mathbb{Q}(z)  +  \int_{\mathcal{W}^*} \frac{1}{r(x)} \mathrm{d}\mathbb{P}(x) - \log r^*(z) \mathrm{d}\mathbb{Q}(z)  - \int_{\mathcal{W}^*} \frac{1}{r^*(x)} \mathrm{d}\mathbb{P}(x)\\
&=\frac{1}{2}\int_{\mathcal{W}^*} \left(\frac{r^*(z)}{r(z)} - 1\right)^2 \mathrm{d}\mathbb{Q}(z) + o\left(\int_{\mathcal{W}^*} \left(\frac{r^*(z)}{r(z)} - 1\right)^2 \mathrm{d}\mathbb{Q}(z)\right).
\end{align*}
By combining them, 
\begin{align*}
    -\left(\widetilde{\mathcal{K}}(\r) - \widetilde{\mathcal{K}}(\rstar)\right)\geq \lambda/2\|r / r^* - 1\|_{\Ltwo(\mathbb{P})}^2 + (1-\lambda)/2\|r^*/r - 1\|_{\Ltwo(\mathbb{Q})}^2  + o(\|r / r^* - 1\|_{\Ltwo(\mathbb{P})}^2) + o(\|r^*/r - 1\|_{\Ltwo(\mathbb{Q})}^2).
\end{align*}
Here, from \Holder's inequality, 
\begin{align*}
    \int_{\mathcal{W}^*} (r(x) - r^*(x))^2\mathrm{d}\mathbb{Q}(x)&=\int_{\mathcal{W}^*} r^{*2}(x)\left(\frac{r(x)}{r^*(x)} - 1\right)^2\mathrm{d}\mathbb{Q}(x)\\
    &=\int_{\mathcal{W}^*} r^{*}(z)\left(\frac{r(z)}{r^*(z)} - 1\right)^2\mathrm{d}\mathbb{P}(z)\leq \| r^{*} \|^2_{L^\infty(\mathbb{P})}\|r / r^* - 1\|_{\Ltwo(\mathbb{P})}^2.
\end{align*}
Similarly,
\begin{align*}
    &\int_{\mathcal{W}^*} (1/r(x) - 1/r^*(x))^2\mathrm{d}\mathbb{Q}(x)\leq \| 1/r^{*} \|^2_{L^\infty(\mathbb{Q})}\|r^* / r - 1\|_{\Ltwo(\mathbb{Q})}^2.
\end{align*}
Therefore,
\begin{equation*}\begin{split}
&\lambda\|r  - r^*\|_{\Ltwo(\mathbb{Q})}^2/ \| r^{*} \|^2_{L^\infty(\mathbb{P})} + (1-\lambda)\|1/r - 1/r^*\|_{\Ltwo(\mathbb{P})}^2/\| 1/r^{*} \|^2_{L^\infty(\mathbb{Q})}\\
&= -  2\left(\widetilde{\mathcal{K}}(\r) - \widetilde{\mathcal{K}}(\rstar)\right)  -  o(\lambda \|r  - r^*\|_{\Ltwo(\mathbb{Q})}^2) - o((1-\lambda) \|1/r - 1/r^*\|_{\Ltwo(\mathbb{P})}^2).
\end{split}\end{equation*}
\end{proof}
Then, we prove Theorem~\ref{thm:est_error_bound} as follows:
\begin{proof}[Proof of Theorem~\ref{thm:est_error_bound}]
Following \citet{Sugiyama2010,Sugiyama:2012:DRE:2181148}, for $\mathbb{P}_1, \mathbb{P}_2 \in\mathcal{P}(\mathcal{W})$, we define unnormalized KL (UKL) objective functional as 
\begin{align*}
   \mathrm{UKL}(g; \mathbb{P}_1, \mathbb{P}_2) = 1 + \int g(x) \mathrm{d}\mathbb{P}_1(x) - \int \exp(g(x)) \mathrm{d}\mathbb{P}_2(x).
\end{align*}
Thanks to the strong convexity, by Lemma~\ref{appdx:sec:appendix:strong-convexity}, we have
\begin{align*}
  &\lambda\|\hat{r}  - r^*\|_{\Ltwo(\mathbb{Q})}^2/ \| r^{*} \|^2_{L^\infty(\mathbb{P})} + (1-\lambda)\|1/\hat{r} - 1/r^*\|_{\Ltwo(\mathbb{P})}^2/\| 1/r^{*} \|^2_{L^\infty(\mathbb{Q})}\\
  &= - \left( \widetilde{\mathcal{K}}(r) - \widetilde{\mathcal{K}}(r^*)\right) -  o(\lambda\|\hat{r}  - r^*\|_{\Ltwo(\mathbb{Q})}^2) - o((1-\lambda)\|1/\hat{r} - 1/r^*\|_{\Ltwo(\mathbb{P})}^2)\\
  &= \lambda \Risk(\hr) - \lambda\Risk(\rstar) + (1-\lambda)\Risk(1/\hr) - (1-\lambda)\Risk(1/\rstar) -  o(\lambda\|\hat{r}  - r^*\|_{\Ltwo(\mathbb{Q})}^2) - o((1-\lambda)\|1/\hat{r} - 1/r^*\|_{\Ltwo(\mathbb{P})}^2).
\end{align*}
Here, we have
\begin{align*}
  &\Risk(\hr) - \Risk(\rstar) \annot{- \hRisk(\hr) + \hRisk(\hr)}{\(=0\)} 
    \annot{- \hRisk(\rstar) + \hRisk(\rstar)}{\(=0\)}\\
  &\leq \annot{(\Risk(\hr) - \Risk(\rstar) + \hRisk(\rstar) - \hRisk(\hr))}{\(=:A\)},
\end{align*}
where we used \(\hRisk(\hr) \leq \hRisk(\rstar)\).

Let us define $\ell_1(r)$ and $\ell_2(r)$ as
\begin{align*}
    &\ell_1(r) := r,\\
    &\ell_2(r) := - \log r.
\end{align*}
For a function $A:\mathcal{W}\to \mathbb{R}$,  observations $\{W_i\}^n_{i=1}$, and a probability measure $\mathbb{W}$, let us denote the expectation and sample average by
\begin{align*}
&\mathbb{E}_{\mathbb{W}}[A(W)] = \int_{\mathcal{W}^*} A(w) \mathrm{d}\mathbb{W}(w),\\
&\widehat{\mathbb{E}}_{\mathbb{W}}[A(W)] = \frac{1}{n}\sum^n_{i=1} A(W_i) = \int_{\mathcal{W}^*} A(w) \mathrm{d}\mathbb{W}_n(w),
\end{align*}
where $\mathbb{W}_b := n^{-1} \sum_{i=1}^n \delta_{W_i}$ is an empirical measure with $\{W_i\}_{i=1}^n$.

To bound \(A\), for ease of notation, let \(\lOne{\r} = \lossOne(\r(X))\) and \(\lTwo{\r} = \lossTwo(\r(X))\).
Then, since
\begin{align*}
  \Risk(\r) &= \Ede\lOneR  + \Enu\lTwoR, \\
  \hRisk(\r) &= \hEde\lOneR  + \hEnu\lTwoR,
\end{align*}
we have
\begin{align*}
  A &= \Risk(\hr) - \Risk(\rstar) + \hRisk(\rstar) - \hRisk(\hr) \\
  &= (\Ede - \hEde)(\lOne{\hr} - \lOne{\rstar})  + (\Enu - \hEnu)(\lTwo{\hr} - \lTwo{\rstar})\\
  &\leq |(\Ede - \hEde)(\lOne{\hr} - \lOne{\rstar})| 
    + |(\Enu - \hEnu)(\lTwo{\hr} - \lTwo{\rstar})|
\end{align*}
By applying Lemma~\ref{appdx:lem:empirical-deviations}, for any \(0 < \gamma < 2\), we have
\[A \leq \Orderp{\vmax{\frac{\|\hr - \rstar\|_{\Ltwo(\mathbb{Q})}^{1 - \gamma/2}}{\sqrt{\min\{n, m\}}}}{\frac{1}{(\min\{n, m\})^{2/(2+\gamma)}}}}.\]

Then, for any \(0 < \gamma < 2\), we get
\begin{align*}
  &\Risk(\hr) - \Risk(\rstar) -  o(\|r  - r^*\|_{\Ltwo(\mathbb{Q})}^2) \\
  &= \Orderp{\vmax{\frac{\|\hr - \rstar\|_{\Ltwo(\mathbb{Q})}^{1 - \gamma/2}}{\sqrt{\min\{n, m\}}}}{\frac{1}{(\min\{n, m\})^{2/(2+\gamma)}}}} -  o(\|\hat{r}  - r^*\|_{\Ltwo(\mathbb{Q})}^2).
  \end{align*}
Similarly, we have
\begin{align*}
  &\Risk(1/\hat{r}) - \Risk(1/r^*)\\ 
  &= \Orderp{\vmax{\frac{\|1/\hr - 1/\rstar\|_{\Ltwo(\mathbb{P})}^{1 - \gamma/2}}{\sqrt{\min\{n, m\}}}}{\frac{1}{(\min\{n, m\})^{2/(2+\gamma)}}}} - o(\|1/\hat{r} - 1/r^*\|_{\Ltwo(\mathbb{P})}^2).
\end{align*}
As a result, we have 
\begin{align*}
&\lambda\|\hat{r}  - r^*\|_{\Ltwo(\mathbb{Q})}^2/ \| r^{*} \|^2_{L^\infty(\mathbb{P})} + (1-\lambda)\|1/\hat{r} - 1/r^*\|_{\Ltwo(\mathbb{P})}^2/\| 1/r^{*} \|^2_{L^\infty(\mathbb{Q})} +  o(\lambda\|\hat{r}  - r^*\|_{\Ltwo(\mathbb{Q})}^2) - o((1-\lambda)\|1/\hat{r} - 1/r^*\|_{\Ltwo(\mathbb{P})}^2)\\
&= -  2\left(\widetilde{\mathcal{K}}(\r) - \widetilde{\mathcal{K}}(\rstar)\right) \\
&=2\lambda\Orderp{\vmax{\frac{\|\hr - \rstar\|_{\Ltwo(\mathbb{Q})}^{1 - \gamma/2}}{\sqrt{\min\{n, m\}}}}{\frac{1}{(\min\{n, m\})^{2/(2+\gamma)}}}}\\
&\ \ \ + 2(1-\lambda)\Orderp{\vmax{\frac{\|1/\hr - 1/\rstar\|_{\Ltwo(\mathbb{P})}^{1 - \gamma/2}}{\sqrt{\min\{n, m\}}}}{\frac{1}{(\min\{n, m\})^{2/(2+\gamma)}}}}.
\end{align*}
Here, we have
\begin{align*}
&\min\left\{1/\| r^{*} \|^2_{L^\infty(\mathbb{P})},1/\| 1/r^{*} \|^2_{L^\infty(\mathbb{Q})} \right\}\max\left\{\lambda\|\hat{r}  - r^*\|_{\Ltwo(\mathbb{Q})}^2, (1-\lambda)\|1/\hat{r} - 1/r^*\|_{\Ltwo(\mathbb{P})}^2\right\}\\
&\ \ \ + o(\max\left\{\lambda\|\hat{r}  - r^*\|_{\Ltwo(\mathbb{Q})}^2, (1-\lambda)\|1/\hat{r} - 1/r^*\|_{\Ltwo(\mathbb{P})}^2\right\})\\
&=\Orderp{\max\left\{\max\left\{\frac{\lambda\|\hr - \rstar\|_{\Ltwo(\mathbb{Q})}^{1 - \gamma/2}}{\sqrt{\min\{n, m\}}}, \frac{(1-\lambda)\|1/\hr - 1/\rstar\|_{\Ltwo(\mathbb{P})}^{1 - \gamma/2}}{\sqrt{\min\{n, m\}}} \right\}, \frac{1}{(\min\{n, m\})^{2/(2+\gamma)}}\right\}}.
\end{align*}
Consider a case where $\lambda\|\hat{r}  - r^*\|_{\Ltwo(\mathbb{Q})}^2 \geq (1-\lambda)\|1/\hat{r} - 1/r^*\|_{\Ltwo(\mathbb{P})}^2$. In this case, we consider
\begin{align*}
&\min\left\{1/\| r^{*} \|^2_{L^\infty(\mathbb{P})},1/\| 1/r^{*} \|^2_{L^\infty(\mathbb{Q})} \right\}\lambda\|\hat{r}  - r^*\|_{\Ltwo(\mathbb{Q})}^2 + o(\lambda\|\hat{r}  - r^*\|_{\Ltwo(\mathbb{Q})}^2)\\
&=\Orderp{\max\left\{\frac{\lambda\|\hr - \rstar\|_{\Ltwo(\mathbb{Q})}^{1 - \gamma/2}}{\sqrt{\min\{n, m\}}}, \frac{1}{(\min\{n, m\})^{2/(2+\gamma)}}\right\}}.
\end{align*}
Without loss of generality, we only consider a case where $\lambda > 0$ and $\min\left\{1/\| r^{*} \|^2_{L^\infty(\mathbb{P})},1/\| 1/r^{*} \|^2_{L^\infty(\mathbb{Q})} \right\} > 0$. Then, because $\min\left\{1/\| r^{*} \|^2_{L^\infty(\mathbb{P})},1/\| 1/r^{*} \|^2_{L^\infty(\mathbb{Q})} \right\}$ and $\lambda$ are constants, either 
\begin{align*}
&\|\hat{r}  - r^*\|_{\Ltwo(\mathbb{Q})}^2 + o(\|\hat{r}  - r^*\|_{\Ltwo(\mathbb{Q})}^2)=\Orderp{\frac{\|\hr - \rstar\|_{\Ltwo(\mathbb{Q})}^{1 - \gamma/2}}{\sqrt{\min\{n, m\}}}},
\end{align*}
or 
\begin{align*}
&\|\hat{r}  - r^*\|_{\Ltwo(\mathbb{Q})}^2 + o(\|\hat{r}  - r^*\|_{\Ltwo(\mathbb{Q})}^2)=\Orderp{\frac{1}{(\min\{n, m\})^{2/(2+\gamma)}}},
\end{align*}
holds. From the first case, we have the following result:
\begin{align*}
&\|\hat{r}  - r^*\|_{\Ltwo(\mathbb{Q})} + o(\|\hat{r}  - r^*\|_{\Ltwo(\mathbb{Q})})=\Orderp{\frac{1}{\min\{n, m\}^{1/(2+\gamma)}}}.
\end{align*}
From the second case, we have the following result:
\begin{align*}
&\|\hat{r}  - r^*\|_{\Ltwo(\mathbb{Q})} + o(\|\hat{r}  - r^*\|_{\Ltwo(\mathbb{Q})})=\Orderp{\frac{1}{(\min\{n, m\})^{1/(2+\gamma)}}}.
\end{align*}
In summary,
\begin{align*}
    &\lambda\|\hat{r}  - r^*\|_{\Ltwo(\mathbb{Q})} + o(\|\hat{r}  - r^*\|_{\Ltwo(\mathbb{Q})})=\Orderp{\frac{1}{(\min\{n, m\})^{1/(2+\gamma)}}}.
\end{align*}

Similarly, for a case where $\lambda\|\hat{r}  - r^*\|_{\Ltwo(\mathbb{Q})}^2 < (1-\lambda)\|1/\hat{r} - 1/r^*\|_{\Ltwo(\mathbb{P})}^2$, we have
\begin{align*}
    &(1-\lambda)\|\hat{r}  - r^*\|_{\Ltwo(\mathbb{P})} + o(\|\hat{r}  - r^*\|_{\Ltwo(\mathbb{P})})=\Orderp{\frac{1}{(\min\{n, m\})^{1/(2+\gamma)}}}.
\end{align*}
By combining them, 
\begin{align*}
    \max\left\{\lambda\|\hat{r}  - r^*\|_{\Ltwo(\mathbb{Q})}, (1-\lambda)\|1/\hat{r} - 1/r^*\|_{\Ltwo(\mathbb{P})}\right\} = \Orderp{\frac{1}{(\min\{n, m\})^{1/(2+\gamma)}}}.
\end{align*}
\end{proof}

Each lemma used in the proof is provided as follows.

\subsection{Bounding the Empirical Deviations}
Following is a proposition originally presented in \citet{vandeGeerEmpirical2000}, which was rephrased in \citet{KanamoriStatistical2012a} in a form that is convenient for our purpose.
\begin{lemma}[Lemma~5.13 in \citet{vandeGeerEmpirical2000}, Proposition~1 in \citet{KanamoriStatistical2012a}]\label{appdx:lem:van-de-geer}
  Let \(\mathcal{F} \subset \Ltwo(\mathbb{P})\) be a function class and the map \(I(f)\)
  be a complexity measure of \(f \in \mathcal{F}\), where \(I\) is a
  non-negative function on \(\mathcal{F}\) and \(I(f_0) < \infty\) for a fixed
  \(f_0 \in \mathcal{F}\). We now define \(\mathcal{F}_M = \{f \in \mathcal{F} :
  I(f) \leq M\}\) satisfying \(\mathcal{F} = \bigcup_{M \geq 1} \mathcal{F}_M\).
  Suppose that there exist \(c_0 > 0\) and \(0 < \gamma < 2\) such that
  \[\sup_{f \in \mathcal{F}_M} \|f - f_0\| \leq c_0 M, \ \sup_{\stackrel{f \in
        \mathcal{F}_M}{\|f - f_0\|_{\Ltwo(P)} \leq \delta}} \|f - f_0\|_\infty
    \leq c_0 M, \quad \text{for all } \delta > 0,\]
  and that \(H_B(\delta, \mathcal{F}_M, \mathbb{P}) = \Order{M/\delta}^\gamma\).
  Then, we have
  \[\sup_{f \in \mathcal{F}} \frac{\left| \int (f - f_0)\mathrm{d}(\mathbb{P} - \mathbb{P}_n)
      \right|}{D(f)} = \Orderp{1}, \ (n \to \infty),\]
  where \(D(f)\) is defined by
  \[D(f) = \vmax{\frac{\|f - f_0\|_{\Ltwo(\mathbb{P})}^{1 - \gamma/2}I(f)^{\gamma/2}}{\sqrt{n}}}{\frac{I(f)}{n^{2/(2+\gamma)}}}.\]
\end{lemma}

\begin{lemma}[Lemma~10 in \citet{KatoTeshima2021}]\label{appdx:lem:empirical-deviations}
  Under the conditions of Theorem~\ref{thm:est_error_bound},
  for any \(0 < \gamma < 2\), we have
  \begin{align*}
    |(\Ede - \hEde)(\lOne{\hr} - \lOne{\rstar})| &= \Orderp{\vmax{\frac{\|\hr - \rstar\|_{\Ltwo(\mathbb{Q})}^{1 - \gamma/2}}{\sqrt{m}}}{\frac{1}{m^{2/(2+\gamma)}}}}\\
    |(\Enu - \hEnu)(\lTwo{\hr} - \lTwo{\rstar})| &= \Orderp{\vmax{\frac{\|\hr - \rstar\|_{\Ltwo(\mathbb{Q})}^{1 - \gamma/2}}{\sqrt{n}}}{\frac{1}{n^{2/(2+\gamma)}}}}\\
  \end{align*}
  as \(n, m \to \infty\).
\end{lemma}

\subsection{Complexity of the hypothesis class}
For the function classes in Definition~\ref{appdx:sparse-network-function-class}, we have the following evaluations of their complexities.

\begin{lemma}[Lemma~5 in \citet{Schmidt-HieberNonparametric2020}]\label{appdx:sparse-network-entropy}
  For \(L \in \Na\) and \(p \in \Na^{L+2}\), let \(V := \prod_{l=0}^{L+1}(p_l + 1)\). Then, for any \(\delta > 0\),
  \[\log \mathcal{N}(\delta, \rClass(L, p, s, \infty), \|\cdot\|_\infty) \leq (s+1)\log(2 \delta^{-1} (L+1) V^2).\]
\end{lemma}

\begin{definition}[ReLU neural networks; \citealp{Schmidt-HieberNonparametric2020}]\label{appdx:sparse-network-function-class}
  For \(L \in \Na\) and \(p = (p_0, \ldots, p_{L+1}) \in \Na^{L+2}\),
  \begin{align*}
    \mathcal{F}(L, p) := &\{f: x \mapsto W_L\sigma_{v_L}W_{L-1}\sigma_{v_{L-1}} \cdots W_1 \sigma_{v_1} W_0 x :\\
                         &\qquad\qquad W_i \in \Re^{p_{i+1} \times p_i}, v_i \in \Re^{p_i} (i = 0, \ldots, L)\},
  \end{align*}
  where \(\sigma_v(y) := \sigma(y - v)\), and \(\sigma(\cdot) = \max\{\cdot, 0\}\) is applied in an element-wise manner.
  Then, for \(s \in \Na, F \geq 0, L \in \Na\), and \(p \in \Na^{L+2}\), define
  \begin{align*}
    \rClass(L, p, s, F) := \{f \in \mathcal{F}(L, p): \sum_{j=0}^L \lzeroNrm{W_j} + \lzeroNrm{v_j} \leq s, \|f\|_\infty \leq F\},
  \end{align*}
  where \(\lzeroNrm{\cdot}\) denotes the number of non-zero entries of the matrix or the vector, and \(\|\cdot\|_\infty\) denotes the supremum norm.
  Now, fixing \(\Lmax, \pmax, s \in \Na\) as well as \(F >0\), we define
  \[\IndLP := \{(L, p): L \in \Na, L \leq \Lmax, p \in [\pmax]^{L+2}\},\]
  and we consider the hypothesis class
  \begin{align*}
    \bar \rClass &:= \bigcup_{(L, p) \in \IndLP}\rClass(L, p, s, F) \\
    \rClass &:= \{r \in \bar\rClass: \mathrm{Im}(r) \subset \rClassRangeTwo\}.
  \end{align*}

  Moreover, we define \(I_1: \IndLP \to \Re\) and \(I: \rClass \to [0, \infty)\) by
  \begin{align*}
    I_1(L, p) &:= 2|\IndLP|^{\frac{1}{s+1}} (L+1) V^2,\\
    I(\r) &:= \max\left\{\|r\|_\infty, \min_{\stackrel{(L, p) \in \IndLP}{r \in \rClassLP}} I_1(L, p)\right\},
  \end{align*}
  where \(V := \prod_{l=0}^{L+1}(p_l + 1)\),
  and we define
  \[\rClassM := \{\r \in \rClass: I(\r) \leq M\}.\]
\end{definition}

% Let \(\bracketEntropy{\delta}{\ell \circ \rClassM}{\|\cdot\|_{\Ltwo(P)}}\) denote the bracketing entropy of \(\ell \circ \rClassM\) with respect to a distribution \(P\).
% Then, for any distribution \(P\), any \(\gamma > 0\), and any \(\delta > 0\), we have
\begin{proposition}[Lemma~8 in \citet{KatoTeshima2021}]\label{appdx:sparse-network-complexity-bounds}
  There exists \(\rClassMSupNormBoundConst > 0\) such that for any \(\gamma > 0\), any \(\delta > 0\), and any \(M \geq 1\), we have
  \begin{align*}
    \log \mathcal{N}\left(\delta, \rClassM, \|\cdot\|_\infty\right) &\leq \frac{s+1}{\gamma} \left( \frac{M}{\delta} \right)^{\gamma}.
  \end{align*}
  and
  \begin{align*}
    \sup_{\r \in \rClassM} \|\r - \rstar\|_\infty &\leq \rClassMSupNormBoundConst M.
  \end{align*}
\end{proposition}

\begin{definition}[Derived function class and bracketing entropy]
  Given a real-valued function class \(\mathcal{F}\),
  define \(\ell \circ \mathcal{F} := \{\ell \circ f: f \in \mathcal{F}\}\).
  By extension, we define \(I: \elled\rClass \to [1, \infty)\) by \(I(\elled\r) = I(r)\) and \(\elled\rClassM := \{\elled\r : \r \in \rClassM\}\).
  Note that, as a result, \(\elled\rClassM\) coincides with \(\{\elled\r \in \elled\rClass: I(\elled\r) \leq M\}\).
\end{definition}

\begin{proposition}[Lemma~9 in \citet{KatoTeshima2021}]\label{appdx:lem:elled-sparse-network-complexity}
  Let \(\ell: (0,\infty) \to \Re\) be a \(\ellLip\)-Lipschitz continuous function.
  Let \(\bracketEntropy{\delta}{\mathcal{F}}{\|\cdot\|_{\Ltwo(\mathbb{P})}}\) denote the bracketing entropy of \(\mathcal{F}\) with respect to a distribution $\mathbb{P}$.
  Then, for any distribution $\mathbb{P}$, any \(\gamma > 0\), any \(M \geq 1\), and any \(\delta > 0\), we have
  \begin{align*}
    \bracketEntropy{\delta}{\ell \circ \rClassM}{\|\cdot\|_{\Ltwo(\mathbb{P})}} &\leq \frac{(s+1)(2\ellLip)^{\gamma}}{\gamma} \left(\frac{M}{\delta} \right)^{\gamma}.
  \end{align*}
  Moreover, there exists \(\rClassMSupNormBoundConst > 0\) such that for any \(M \geq 1\) and any distribution \(P\),
  \begin{align*}
    \sup_{\elled\r \in \elled\rClassM} \|\elled\r - \elled\rstar\|_{\Ltwo(\mathbb{P})} &\leq \rClassMSupNormBoundConst\ellLip M, \\
    \sup_{\stackrel{\elled\r \in \elled\rClassM}{\|\elled\r - \elled\rstar\|_{\Ltwo(\mathbb{P})} \leq \delta}} \|\elled\r - \elled\rstar\|_\infty &\leq \rClassMSupNormBoundConst\ellLip M, \quad \text{for all } \delta > 0.
  \end{align*}
\end{proposition}

\section{Proof of Lemma~\ref{lem:temp_drm_kl}}
\label{appdx:lem:temp_drm_kl}
\begin{proof}
  %For $\lambda \in \{0,1\}$, $\mathrm{DRM}^\lambda_{\mathcal{R}}(\mathbb{P}\parallel\mathbb{Q})$ corresponds to the Kullback-Leibler divergence as Theorem \ref{thm:drm_kl}, hence the statement holds.

  We study two cases: (i) $\mathcal{W}_p=\mathcal{W}_q$, and (ii) $\mathcal{W}_p \neq \mathcal{W}_q$.

  Consider the first case that $\mathcal{W}_p = \mathcal{W}_q$ holds.
  By Theorem~\ref{thm:strat_silver2}, $\tilde{r}_{\mathrm{strat}}=r^*$ attains the maximum of $\mathrm{DRM}^\lambda_{\mathcal{R}}(\mathbb{P} \|\mathbb{Q})$.
  Hence, we have
  \begin{align}
      \mathrm{DRM}^\lambda_{\mathcal{R}}(\mathbb{P} \| \mathbb{Q}) = \lambda \mathrm{KL}(\mathbb{P} \| \mathbb{Q}) + (1-\lambda) \mathrm{KL}(\mathbb{Q} \| \mathbb{P}). \label{def:drm_kl}
  \end{align}
  Since the Kullback-Leibler divergence $\mathrm{KL}(\mathbb{P} \| \mathbb{Q})$ satisfies $\mathrm{KL}(\mathbb{P} \| \mathbb{Q}) = 0 \Leftrightarrow \mathbb{P} = \mathbb{Q}$, we obtain the statement.

  Consider the second case $\mathcal{W}_q \neq \mathcal{W}_q$.
  In this case, we always have $\mathbb{P} \neq \mathbb{Q}$, hence it is sufficient to show that $\mathrm{DRM}^\lambda_{\mathcal{R}}(\mathbb{P} \|\mathbb{Q}) > 0$.
  We substitute $r^\dagger_{\mathrm{strat}}$ and obtain
  \begin{align*}
      \mathrm{DRM}^\lambda_{\mathcal{R}}(\mathbb{P} \| \mathbb{Q}) &\geq \mathcal{K}(r^\dagger_{\mathrm{strat}}) \\
      &=  \lambda \int_{\mathcal{W}^*} \log r^*(x) \mathrm{d} \mathbb{P}(x) + (1-\lambda) \int_{\mathcal{W}^*} \log (1/r^*(x)) \mathrm{d} \mathbb{Q}(x)\\
      & \quad +\lambda \int_{\mathcal{W}_p \backslash \mathcal{W}_q} \log \overline{R} \mathrm{d} \mathbb{P}(x) + (1-\lambda) \int_{\mathcal{W}_q \backslash \mathcal{W}_p} \log \overline{R} \mathrm{d} \mathbb{Q}(x).
  \end{align*}
We have
\begin{align*}
    \int_{\mathcal{W}^*} \log r^*(x) \mathrm{d} \mathbb{P}(x) &= \int_{\mathcal{W}^*} - \log \left( \frac{q^*(x)}{p^*(x)}  \right) \mathrm{d} \mathbb{P}(x)\\
    & \geq \int_{\mathcal{W}^*} -  \left( \frac{q^*(x)}{p^*(x)} - 1 \right) \mathrm{d} \mathbb{P}(x) \\
    & = \int_{\mathcal{W}^*}  \left(p^*(x) - q^*(x) \right) \mathrm{d}x \\
    &=\mathbb{P}(\mathcal{W}^*) - \mathbb{Q}(\mathcal{W}^*),
\end{align*}
where the inequality follows $\log (x) \leq (x-1)$.
Using this inequality, we continue the lower bound on $\mathrm{DRM}^\lambda_{\mathcal{R}}(\mathbb{P} \| \mathbb{Q})$ as
\begin{align*}
    &\mathrm{DRM}^\lambda_{\mathcal{R}}(\mathbb{P} \| \mathbb{Q}) \\
    & \geq \lambda (\mathbb{P}(\mathcal{W}^*) - \mathbb{Q}(\mathcal{W}^*)) + (1-\lambda) (\mathbb{Q}(\mathcal{W}^*) - \mathbb{P}(\mathcal{W}^*)) \\
    &\quad + (\lambda \mathbb{P}(\mathcal{W}_p \backslash \mathcal{W}_q) +(1 - \lambda) \mathbb{Q}(\mathcal{W}_q \backslash \mathcal{W}_p)) \log \overline{R}\\
    &= (2\lambda - 1)\mathbb{P}(\mathcal{W}^*) + (1-2 \lambda) \mathbb{Q}(\mathcal{W}^*)  + (\lambda \mathbb{P}(\mathcal{W}_p \backslash \mathcal{W}_q) +(1 - \lambda) \mathbb{Q}(\mathcal{W}_q \backslash \mathcal{W}_p)) \log \overline{R}.
\end{align*}
We show that the lower bound is strictly positive.
For $\lambda \in [0,1/2]$, the lower bound is larger than $0$ if 
\begin{align*}
    \log \overline{R} > \frac{(1-2 \lambda) (\mathbb{P}(\mathcal{W}^*) - \mathbb{Q}(\mathcal{W}^*))}{\lambda \mathbb{P}(\mathcal{W}_p \backslash \mathcal{W}_q) +(1 - \lambda) \mathbb{Q}(\mathcal{W}_q \backslash \mathcal{W}_p)}
\end{align*}
holds.
Similarly, for $\lambda \in [1/2,1]$, we obtain the same result when we have
\begin{align*}
    \log \overline{R} > \frac{(2 \lambda - 1) (\mathbb{Q}(\mathcal{W}^*) - \mathbb{P}(\mathcal{W}^*))}{\lambda \mathbb{P}(\mathcal{W}_p \backslash \mathcal{W}_q) +(1 - \lambda) \mathbb{Q}(\mathcal{W}_q \backslash \mathcal{W}_p)}.
\end{align*}
Note that $\min\{\mathbb{P}(\mathcal{W}_p \backslash \mathcal{W}_q),\mathbb{Q}(\mathcal{W}_q \backslash \mathcal{W}_p)\} > 0$ holds by the setting.
Hence, if $\overline{R}$ is sufficiently large such that satisfies the inequalities, we show that $\mathrm{DRM}^\lambda_{\mathcal{R}}(\mathbb{P} \| \mathbb{Q}) > 0$.
\end{proof}

\section{Proof of Theorem~\ref{thm:toplogical}}
\label{appdx:thm:topological}

\begin{proof}
We show the statements one by one.

\textit{1}: We put $r^*$ in the maximum in $\mathrm{DRM}^\lambda_{\mathcal{R}}(\mathbb{P} \| \mathbb{Q})$ and obtain
\begin{align*}
    \mathrm{DRM}^\lambda_{\mathcal{R}}(\mathbb{P} \| \mathbb{Q}) &\geq \lambda \int_{\mathcal{W}_p} \log (r^*(x)) \mathrm{d} \mathbb{P}(x) + (1-\lambda) \int_{\mathcal{W}_q} \log (1/r^*(x)) \mathrm{d} \mathbb{Q}(x) \\
    &= \lambda \mathrm{KL}(\mathbb{P} \| \mathbb{Q}) + ( 1 - \lambda) \mathrm{KL}(\mathbb{Q} \| \mathbb{P}).
\end{align*}
Hence, we have $\limsup_{N \to \infty} \lambda \mathrm{KL}(\mathbb{P}_N \| \mathbb{P}) + (1-\lambda) \mathrm{KL}(\mathbb{P} \| \mathbb{P}_N) = 0$.
Combining the non-negativity of the Kullback-Leibler divergence, we obtain $\lim_{n \to \infty} \mathrm{KL}(\mathbb{P}_N \| \mathbb{P})$ and $\lim_{n \to \infty} \mathrm{KL}(\mathbb{P} \| \mathbb{P}_N) = 0$.
Since the Kullback-Leibler divergence implies weak convergence \citep{gibbs2002choosing} with the bounded assumption, we obtain the statement.

\textit{2}:
The direction $\Rightarrow$ follows the above  first statement.
We show the opposite $\Leftarrow$.
With $\lambda = 1/2$, we obtain
\begin{align*}
    &\mathrm{DRM}^{1/2}_{\mathcal{R}}(\mathbb{P}_N \| \mathbb{P}) = \frac{1}{2} \sup_{r \in C(\mathcal{R})} \left\{ \int_{\mathcal{W}} \log r(x) \mathrm{d}(\mathbb{P}_N - \mathbb{P})(x) \right\}.
\end{align*}
Since $r \in C(\mathcal{R})$ is a continuous, bounded, and strictly positive function, $\log r$ is continuous and bounded.
Hence, by the definition of weak convergence, we obtain the statement.
%We show the statement by contradiction.
%Suppose that there exists a subsequence $\{\mathbb{P}_{n_k}\}_{n_k}$ with $n_k \to \infty$ as $k \to \infty$ such that $\mathrm{DRM}_{\mathrm{R}}^\lambda(\mathbb{P}_{n_k} \| \mathbb{Q}) \geq c $ with some $c > 0$.
%If $\mathcal{W}_{p_{n_k}} = \mathcal{W}_P$ holds, the representation of $\mathrm{DRM}_{\mathcal{R}}^\lambda(\mathbb{P}_{n_k} \| \mathbb{Q})$ by $\mathrm{KL}(\mathbb{P} \| \mathbb{Q})$ and $\mathrm{KL}(\mathbb{Q} \| \mathbb{P})$ as \eqref{def:drm_kl} proves the statement, since the Kullback-Leibler divergence metrizes the weak convergence.
%If $\mathcal{W}_{p_{n_k}} \neq \mathcal{W}_p$ holds, we utilize $r^\dagger_{\mathrm{strat}}$ as the proof of Lemma \ref{lem:temp_drm_kl} and obtain
%\begin{align*}
%    \mathrm{DRM}_{\mathcal{R}}^\lambda(\mathbb{P}_{n_k} \| \mathbb{Q}) > 0,
%\end{align*}
%with sufficiently large $\overline{R}$ as Lemma \ref{lem:temp_drm_kl}.
\end{proof}

\begin{algorithm}[tb]
   \caption{nnDRM}
   \label{alg}
\begin{algorithmic}
   \STATE {\bfseries Input:} Training data $\big\{X_i\big\}^{n}_{i=1}$ and $\big\{Z_j\big\}^{m}_{j=1}$, the algorithm for stochastic optimization such as Adam \citep{kingma2014method}, the learning rate $\gamma$, the regularization coefficient $\lambda$ and function $\mathcal{R}(r)$, and a constant $C > 0$. 
   \STATE {\bfseries Output:} A density ratio estimator $\hat{r}$.
   \WHILE{No stopping criterion has been met:}
   \STATE $N$ mini-batches: $\big\{\big(\big\{X^{k}_i\big\}^{n_{k} }_{i=1}, \big\{Z^{k}_j\big\}^{m_k}_{j=1}\big)\big\}^N_{k=1}.$
   \FOR{$k=1$ to $N$}
   \IF{$\sum^{m_k}_{j=1} r(Z^{k}_j) - C \sum^{n_k}_{i=1} r(X^{k}_i) \geq 0$:}
   \STATE{Gradient decent:} set gradient
  \[\mathrm{Grad} = \nabla_r\widehat{\mathrm{nnUKL}}(r, \mathbb{P}^k_n, \mathbb{Q}^k_n).\]
   \ELSE 
   \STATE{Gradient ascent:} set gradient
   \[\mathrm{Grad} = \nabla_r \big\{\sum^{m_k}_{j=1} r(Z^{k}_j) - C \sum^{n_k}_{i=1} r(X^{k}_i)\big\}.\]
   \ENDIF
   \IF{$\sum^{n_k}_{i=1} \frac{1}{r(X^{k}_i)} - C \sum^{m_k}_{j=1} \frac{1}{r(Z^{k}_j)} \geq 0$:}
   \STATE{Gradient decent:} add gradient
   \[\mathrm{Grad} \pluseq \nabla_r\widehat{\mathrm{nnUKL}}(\frac{1}{r}, \mathbb{Q}^k_n, \mathbb{P}^k_n).\]
   \ELSE 
   \STATE{Gradient ascent:} add gradient 
   \[\mathrm{Grad} \pluseq \nabla_r \big\{\sum^{n_k}_{i=1} \frac{1}{r(X^{k}_i)} - C \sum^{m_k}_{j=1} \frac{1}{r(Z^{k}_j)}\big\}.\]
   \ENDIF
   \STATE Update $r$ with the gradient and the learning rate $\gamma$. 
   \ENDFOR
   \ENDWHILE
\end{algorithmic}
\end{algorithm}

\section{Estimation of the DRM and Density Ratio}
\label{appendix:est}
We can estimate the density ratio by solving the inner maximization problem in DRM; that is, we consider minimizing\begin{align*}
&\widehat{\mathcal{K}}(r) = \Bigg\{ \lambda \frac{1}{n}\sum^n_{i=1} \log r(X_i) - (1-\lambda) \frac{1}{m}\sum^m_{j=1} \log r(Z_j) - \frac{1-\lambda}{n}\sum^n_{i=1} \frac{1}{r(X_i)} - \frac{\lambda}{m}\sum^m_{j=1} r(Z_j)\Bigg\}, 
\end{align*}

Besides, if we know the upper bounds of $r^*$ and $1/r^*$, we can also impose the non-negative correction proposed by \citet{KiryoPositiveunlabeled2017} and \citet{KatoTeshima2021}. In UKL , does not become negative because... Based on this motivation. \citet{KatoTeshima2021} proposes the following nonnegative UKL:
\begin{align*}
&\widehat{\mathrm{nnUKL}}(r):= - \sum^n_{i=1} \Big(\log\big(r(X_i)\big) - Cr(X_i)\Big) +  \left( \sum^m_{j=1} r(Z_j) - C \sum^n_{i=1} r(X_i)\right)_+,
\end{align*}
where $C $. Note that if $\overline{R} = \infty$, $C = 0$ and the second term is always positive. Therefore, the nonnegative UKL is identical to the original UKL. Thus, we can regard nonnegative UKL is a generalization of the original UKL.

The empirical counterpart of the nonnegative UKL is given as
\begin{align*}
    &\widehat{\mathrm{nn}\mathcal{K}}(r) =\lambda \widehat{\mathrm{nnUKL}}(r, \mathbb{P}_n, \mathbb{Q}_m)- (1-\lambda) \widehat{\mathrm{nnUKL}}(1/r, \mathbb{Q}_m. \mathbb{P}_n),
\end{align*}
When the nonnegative correction is violated, instead of simply replacing it with $0$, we can use gradient ascent; that is, if ...
The use of gradient ascent is reported to improve the empirical performance \citet{KiryoPositiveunlabeled2017,KatoTeshima2021}. Our proposed algorithm is summarized in Algorithm~\ref{alg}.

\begin{table*}[t]

\caption{Results of Appendix~\ref{appdx:sec:add_regression_exp}: means, medians, and stds of the squared error in DRM-based DRE using synthetic datasets. The lowest mean and median (med) methods are highlighted in bold.}
\label{table:l2_error_2}
\begin{center}
\scalebox{0.65}{
\begin{tabular}{|r|rrr|rrr|rrr|rrr|rrr|rrr|}
\hline
\multicolumn{19}{|c|}{sample sizes: $n=1000$, $m=100$} \\
\hline
\multirow{2}{*}{dim} &    \multicolumn{3}{c|}{$\lambda = 0.0$}  & \multicolumn{3}{c|}{$\lambda = 0.2$}  & \multicolumn{3}{c|}{$\lambda = 0.4$}   & \multicolumn{3}{c|}{$\lambda = 0.6$}   & \multicolumn{3}{c|}{$\lambda = 0.8$}  & \multicolumn{3}{c|}{$\lambda = 1,0$} \\
 &      mean  &      med  &      std  &      mean  &      med  &      std &      mean  &      med  &      std  &   mean  &      med  &      std &   mean  &      med  &      std &     mean  &      med  &      std \\
\hline
10 &   \textbf{8.129} &   \textbf{6.045} &  8.444 &   8.969 &   6.699 &  9.371 &   9.843 &   7.467 &  9.778 &  10.867 &   8.464 &  9.954 &  11.915 &   9.384 &  10.315 &  12.874 &  10.224 &  10.584 \\
100 &  26.693 &  26.407 &  6.058 &  18.115 &  16.749 &  5.663 &  \textbf{12.581} &  \textbf{12.130} &  3.165 &  14.481 &  13.816 &  3.439 &  16.296 &  15.637 &   3.520 &  17.145 &  16.438 &   3.554 \\
\hline
\multicolumn{19}{|c|}{} \\
\hline
\multicolumn{19}{|c|}{sample sizes: $n=10000$, $m=100$} \\
\hline
\multirow{2}{*}{dim} &    \multicolumn{3}{c|}{$\lambda = 0.0$}  & \multicolumn{3}{c|}{$\lambda = 0.2$}  & \multicolumn{3}{c|}{$\lambda = 0.4$}   & \multicolumn{3}{c|}{$\lambda = 0.6$}   & \multicolumn{3}{c|}{$\lambda = 0.8$}  & \multicolumn{3}{c|}{$\lambda = 1,0$} \\
 &      mean  &      med  &      std  &      mean  &      med  &      std &      mean  &      med  &      std  &   mean  &      med  &      std &   mean  &      med  &      std &     mean  &      med  &      std \\
\hline
10 &  \textbf{3.458} &  \textbf{3.338} &  1.202 &  3.767 &  3.729 &  1.143 &  4.088 &  4.148 &  1.146 &   4.669 &   4.562 &  1.208 &   5.518 &   5.608 &  1.134 &   6.71 &   6.479 &  1.167 \\
100 &  \textbf{6.748} &  \textbf{6.608} &  1.138 &  6.886 &  6.759 &  1.213 &  8.157 &  8.028 &  1.348 &  10.579 &  10.329 &  1.443 &  12.502 &  12.303 &  1.522 &  13.71 &  13.578 &  1.506 \\
\hline
\multicolumn{19}{|c|}{} \\
\hline
\multicolumn{19}{|c|}{sample sizes: $n=10000$, $m=100$} \\
\hline
\multirow{2}{*}{dim} &    \multicolumn{3}{c|}{$\lambda = 0.0$}  & \multicolumn{3}{c|}{$\lambda = 0.2$}  & \multicolumn{3}{c|}{$\lambda = 0.4$}   & \multicolumn{3}{c|}{$\lambda = 0.6$}   & \multicolumn{3}{c|}{$\lambda = 0.8$}  & \multicolumn{3}{c|}{$\lambda = 1,0$} \\
 &      mean  &      med  &      std  &      mean  &      med  &      std &      mean  &      med  &      std  &   mean  &      med  &      std &   mean  &      med  &      std &     mean  &      med  &      std \\
\hline
0 &   \textbf{6.735} &   \textbf{6.445} &  1.537 &   7.221 &   7.134 &  1.750 &   8.069 &   7.931 &  1.576 &   9.099 &   8.999 &  1.558 &  10.148 &  10.026 &  1.503 &  10.975 &  10.854 &  1.531 \\
1 &  15.215 &  14.857 &  2.579 &  11.843 &  11.674 &  1.304 &  \textbf{11.556} &  \textbf{11.434} &  1.279 &  14.722 &  14.514 &  1.288 &  16.499 &  16.329 &  1.267 &  17.267 &  17.206 &  1.293 \\
\hline
\end{tabular}
}
\end{center}
\end{table*}

\section{Additional Results of Section~\ref{sec:regression_exp}}
\label{appdx:sec:add_regression_exp}
In addition to the experimental results shown in Section~\ref{sec:regression_exp}, we investigate the performance of the DRM-based DRE with different $\lambda$, chosen from $\{0.0, 0.2, 0.4, 0.6, 0.8, 1.0\}$.

First, we change the sample sizes. We show the results with sample sizes $(n,m)=(1000, 100)$,  $(n,m)=(10000, 1000)$, and  $(n,m)=(10000, 100)$ in Table~\ref{table:l2_error_2}. We choose the dimension $d$ from $\{10, 100\}$. The other settings are identical to that of Section~\ref{sec:regression_exp}. In Section~\ref{sec:regression_exp}, DRM with $\lambda=0.9$ achieves the lowest mean and squared errors. However, in this result, DRM with $\lambda=0.5$ achieves lower mean and squared errors than that with $\lambda=0.9$. We consider that this is because balancing $lambda$ between the log likelihood of the density ratio and inverse density ratio makes the estimation error lower as discussed in \citet{Wooldridge2001}. In this case, because $n$ is larger than $m$. Therefore, weighting the log likelihood $\frac{1}{n}\sum^n_{i=1}\log r(X_i)$ more than $\frac{1}{m}\sum^m_{j=1}\log \frac{1}{r(X_i)}$ may make the estimation more accurate. 

Next, we change the mean vectors from the setting of Section~\ref{sec:regression_exp} as $\mu_p = (1,0,\dots, 0)^\top$ and $\mu_q = (0,0,\dots, 0)^\top$. The other settings are the same as that of Section~\ref{sec:regression_exp}. The results shown by boxplots in Figures~\ref{fig:box_plot1} and \ref{fig:box_plot2}.

It can be confirmed that the error can be reduced by adjusting $\lambda$ appropriately. For example, in Section~\ref{sec:regression_exp}, DRM with $\lambda=0.1$ shows the best performance. However, by observing the results carefully, we can find that there are cases where setting $lambda$ around $0,4$ may reduce the error more than setting with $\lambda=0$. We can also find that appropriate choices of $\lambda$ are also affected by the changes in the sample size ratio and the mean vectors.

\begin{table*}[t]
\centering
\caption{Negative Log-likelihood (NLL) and MMD, multiplied by $10^3$, results on six 2-d synthetic datasets. Lower is better.}
\vskip 0.15in
\label{tab:synthetic}
\scalebox{0.75}{
\begin{tabular}{|c|c|cccccc|}
\hline
Metric & GAN & MoG & Banana & Rings & Square & Cosine & Funnel \\
\hline
 \multirow{2}{*}{NLL}  &  WGAN  & $\bm{-2.59 \pm 0.04}$  & $-3.58 \pm 0.00$  & $\bm{-4.26 \pm 0.00}$  & $-3.74 \pm 0.00$  & $-3.99 \pm 0.00$  & $\bm{-3.58 \pm 0.01}$ \\
& KL-WGAN & $-2.55 \pm 0.01$  & $\bm{-3.59 \pm 0.01}$  & $-4.25 \pm 0.01$  & $\bm{-3.73 \pm 0.01}$  & $\bm{-4.00 \pm 0.02}$  & $-3.57 \pm 0.01$ \\
& SLoGAN & $-2.52 \pm 0.01$  & $-3.58 \pm 0.00$  & $\bm{-4.26 \pm 0.00}$  & $-3.71 \pm 0.01$  & $-3.99 \pm 0.00$  & $-3.57 \pm 0.00$  \\\hline
\multirow{2}{*}{MMD} & WGAN & $16.38 \pm 7.47$  & $2.32 \pm 0.80$  & $1.73 \pm 0.30$  & $1.48 \pm 0.39$  & $1.08 \pm 0.32$  & $1.83 \pm 0.72$ \\
& KL-WGAN & $\bm{2.26 \pm 0.22}$  & $1.96 \pm 0.12$  & $\bm{1.34 \pm 0.23}$  & $1.35 \pm 0.21$  & $\bm{1.00 \pm 0.13}$  & $\bm{1.16 \pm 0.31}$ \\
& SLoGAN & $5.89 \pm 1.49$  & $\bm{1.38 \pm 0.43}$  & $1.79 \pm 0.37$  & $\bm{0.73 \pm 0.09}$  & $1.10 \pm 0.24$  & $1.63 \pm 0.32$ \\\hline
\end{tabular}}
\end{table*}

\section{Experimental on Distribution Modeling} 
\label{appdx:generate}
We investigate the distribution modeling using DRM. Following \citet{Song2020}, we use the 2-d synthetic datasets include Mixture of Gaussians (MoG), Banana, Ring, Square, Cosine and Funnel; these datasets cover different modalities and geometries. We compare our proposed DRM with the WGAN and KL-WGAN, proposed by \citet{Song2020}.

After training, we draw 5,000 samples from the generator and then evaluate two metrics over a fixed validation set. One is the negative log-likelihood (NLL) of the validation samples on a kernel density estimator fitted over the generated samples; the other is the MMD (\citet{borgwardt2006integrating}) between the generated samples and validation samples. To ensure a fair comparison, we use identical kernel bandwidths for all cases.

\paragraph{Distribution modeling.}
We report the mean and standard error for the NLL and MMD results in Tables~\ref{tab:synthetic} (with 5 random seeds in each case). We illustrate the histograms of samples in Figure~\ref{fig:hist2d}. 
%, we can visually observe where our SLoGAN performs significantly better than WGAN. For example, WGAN fails to place enough probability mass in the center of the Gaussians in MoG and fails to learn a proper square in Square, unlike our KL-WGAN approaches.

\paragraph{Density ratio estimation.} We demonstrate that SLoGAN learns the density ratio simultaneously. We consider measuring the density ratio from synthetic datasets, and compare them with the the discriminators of WGAN, $f$-GAN with KL divergence, KL-WGAN. We evaluate the density ratio estimation quality by multiplying $\mathrm{d}\mathbb{Q}$ with the estimated density ratios, and compare that with the density of $\mathbb{P}$; ideally the two quantities should be identical. We demonstrate empirical results in Figure~\ref{fig:density-ratio}, where we plot the samples used for training, the ground truth density $p^*$ and the two estimates given by two methods. In terms of estimating density ratios, our proposed approach estimates it as well as f-GAN and KL-WGAN.

\paragraph{Stability of discriminator objectives.}
For the MoG and Square and Cosine datasets, we further show the estimated divergences over a batch of 256 samples in Figure~\ref{fig:div_curve}. While divergences of KL-WGAN and our proposed SLoGAN  decrease more stable tan that of WGAN.

\clearpage

\begin{figure*}[htbp]
\begin{center}
    \includegraphics[width=100mm]{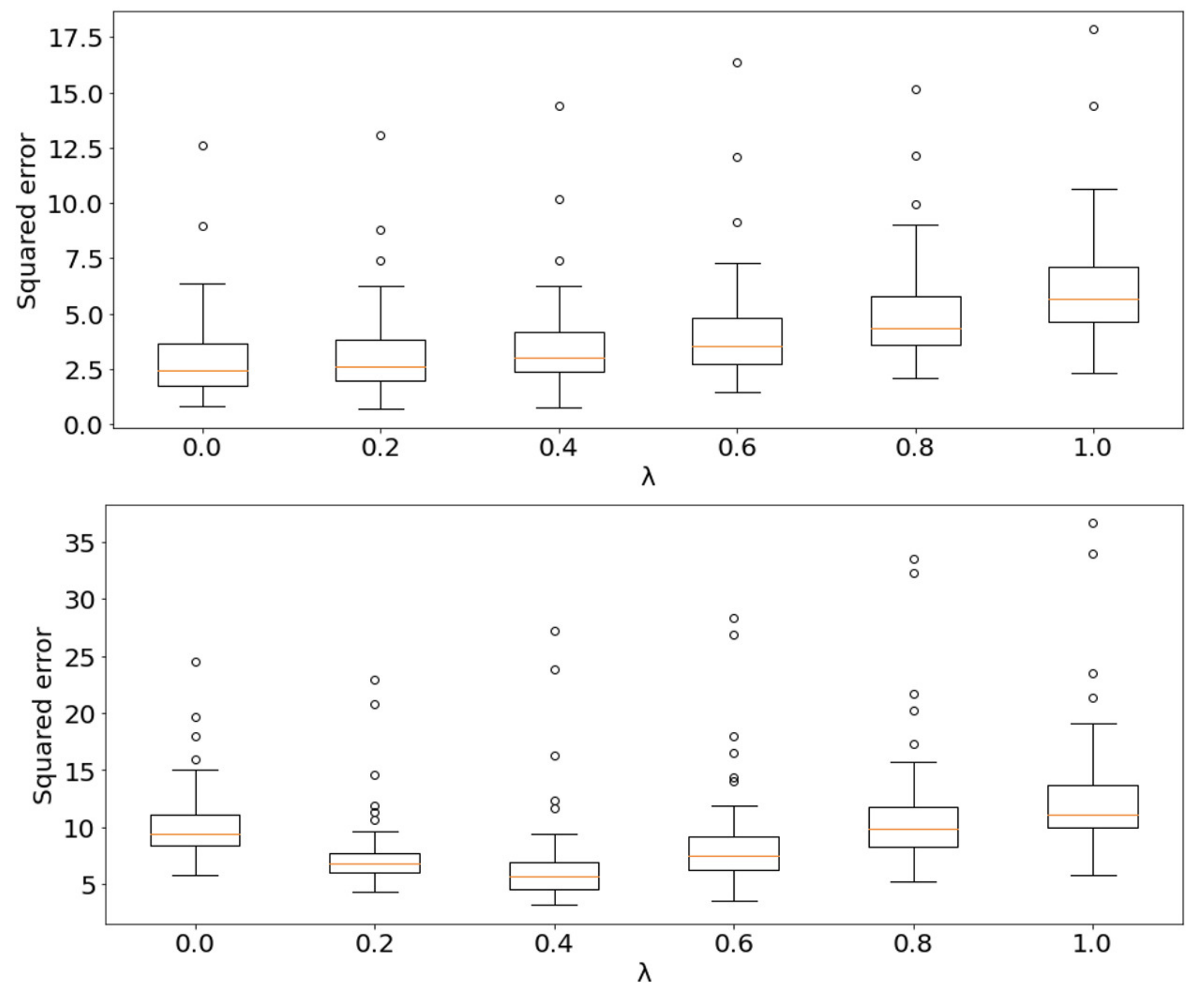}
\end{center}
\caption{Results of Appendix~\ref{appdx:sec:add_regression_exp}: The mean vectors are $\mu_p = (0,0,\dots, 0)^\top$ and $\mu_q = (1,0,\dots, 0)^\top$.}
\label{fig:box_plot1}

\begin{center}
    \includegraphics[width=100mm]{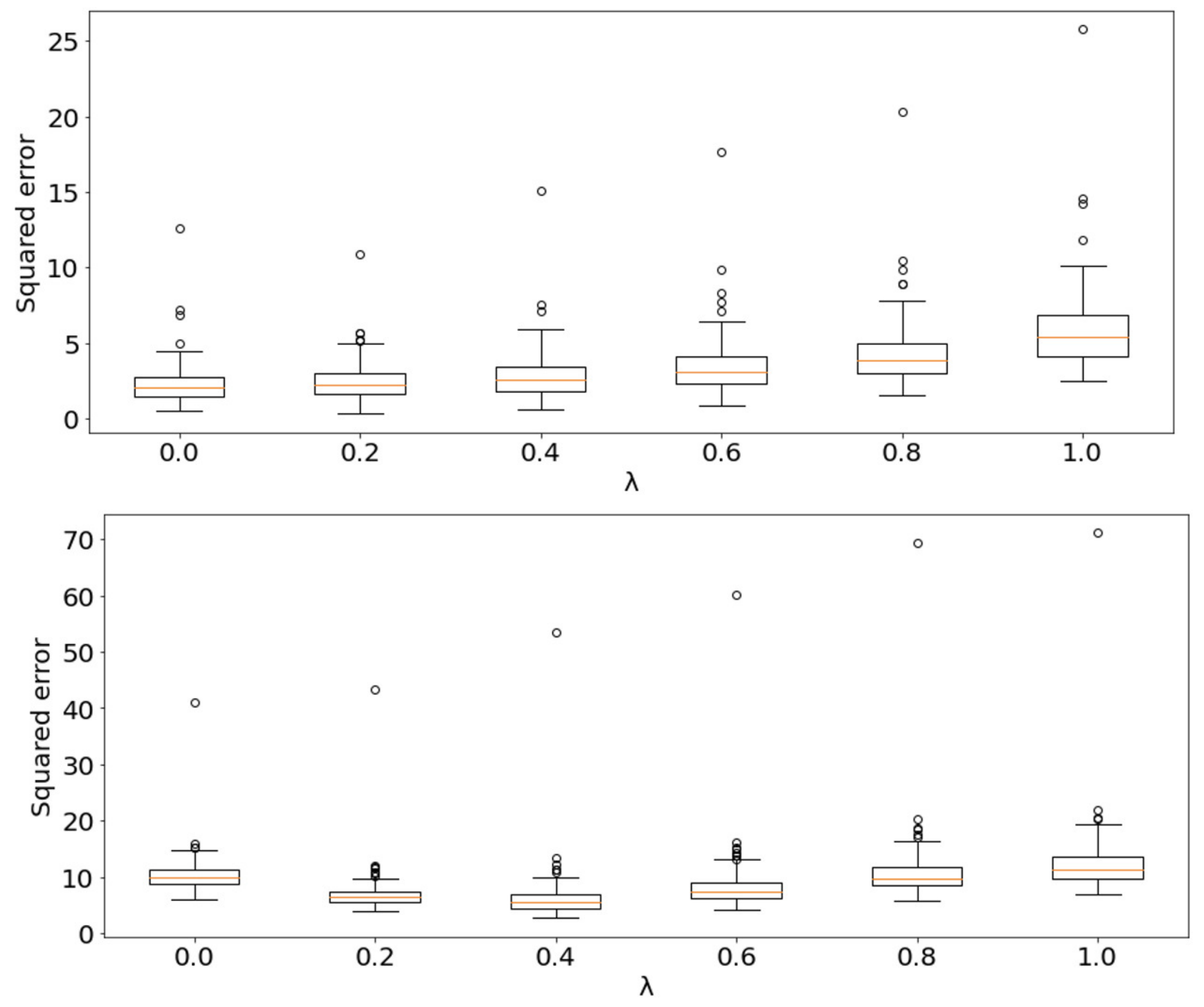}
\end{center}
\caption{Results of Appendix~\ref{appdx:sec:add_regression_exp}: The mean vectors are $\mu_p = (1,0,\dots, 0)^\top$ and $\mu_q = (0,0,\dots, 0)^\top$.}
\label{fig:box_plot2}
\end{figure*}

\begin{figure}[htbp]
\begin{center}
    \includegraphics[width=110mm]{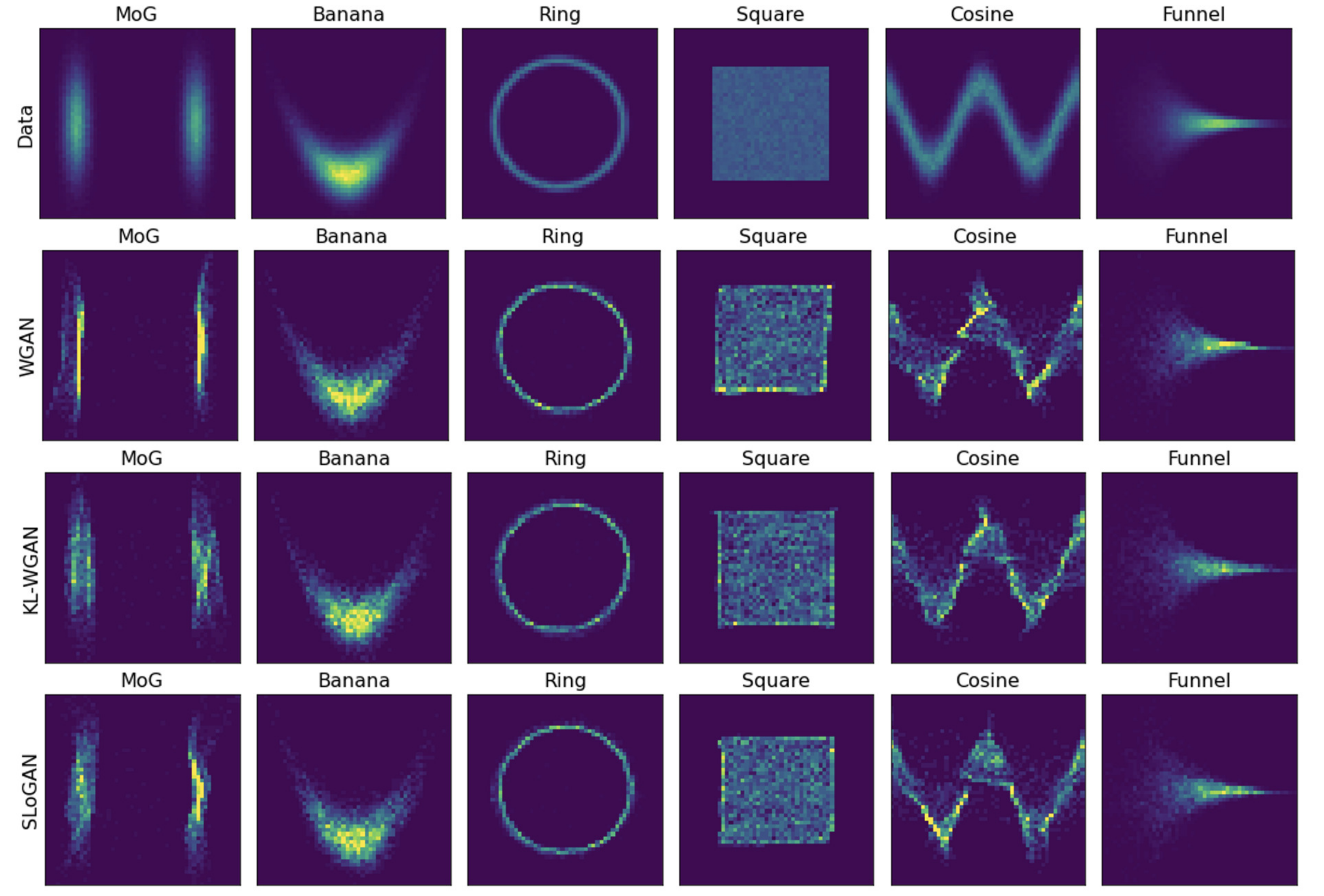}
\end{center}
\caption{Results of Appendix~\ref{appdx:generate}: Histograms of samples from the true data distribution, WGAN and KL-WGAN, and our SLoGAN}
\label{fig:hist2d}

\begin{center}
    \includegraphics[width=110mm]{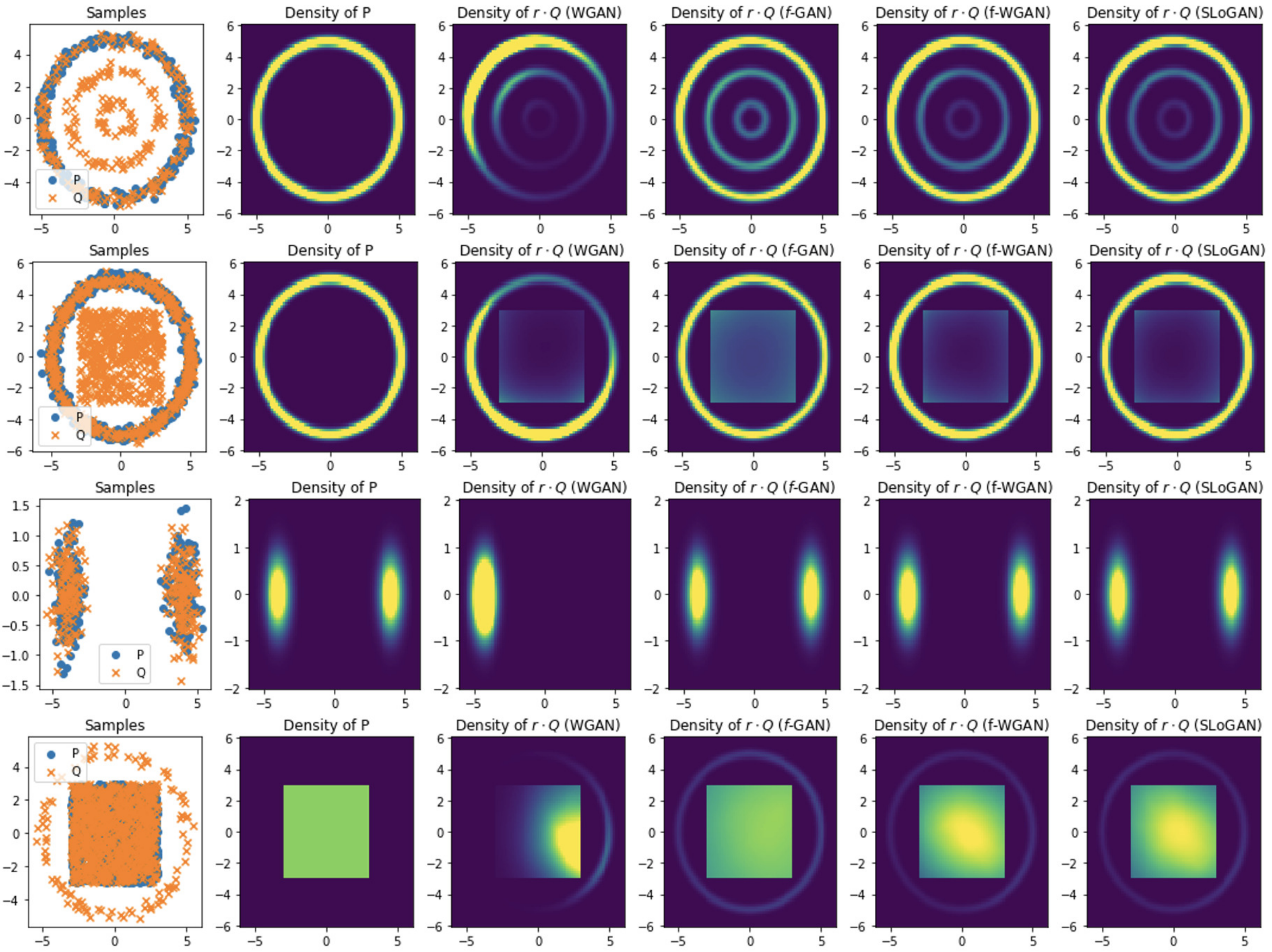}
\end{center}
\caption{Results of Appendix~\ref{appdx:generate}: Estimating density ratios. The first column contains the
samples used for training, the second column is the ground truth
density of $\mathbb{P}$ ($p^*$), the third and sixth columns are the density of $\mathbb{Q}$
times the estimated density ratios from WGAN (thrid
column), f-GAN (fourth
column), KL-WGAN (fifth column), and our SLoGAN (sixth column).}
\label{fig:density-ratio}

\begin{center}
    \includegraphics[width=115mm]{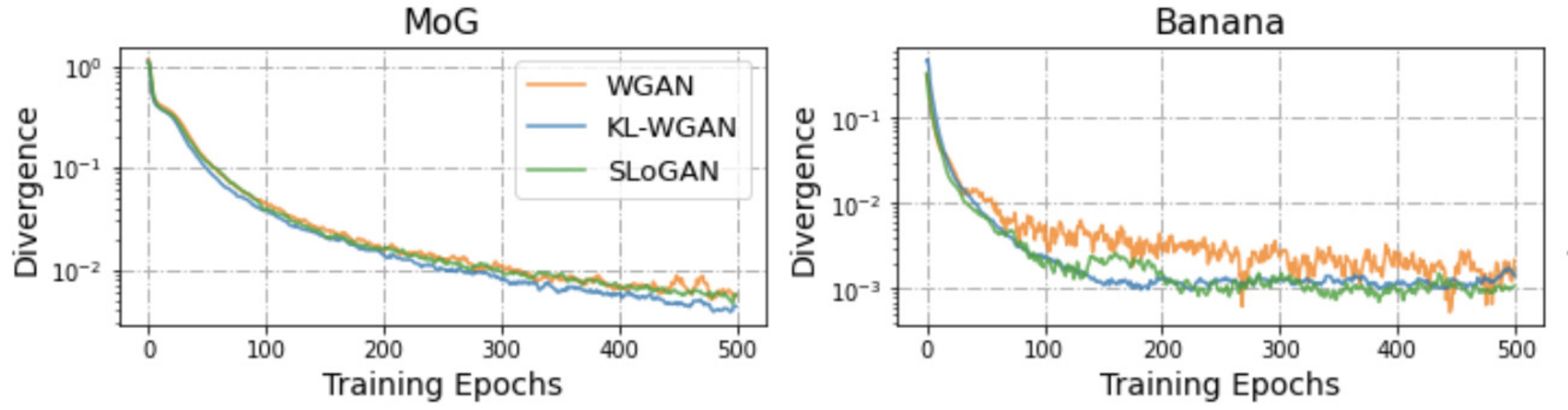}
\end{center}
\caption{Results of Appendix~\ref{appdx:generate}: Estimated divergence with respect to training epochs (smoothed with a window of $10$).}
\label{fig:div_curve}
\end{figure}

\end{document}

%% file: math_commands.tex
\usepackage{amsmath,amsfonts,bm}

\def\eqref#1{equation~\ref{#1}}
\def\1{\bm{1}}

\DeclareMathAlphabet{\mathsfit}{\encodingdefault}{\sfdefault}{m}{sl}
\SetMathAlphabet{\mathsfit}{bold}{\encodingdefault}{\sfdefault}{bx}{n}

\newcommand{\E}{\mathbb{E}}

\DeclareMathOperator*{\argmax}{arg\,max}
\DeclareMathOperator*{\argmin}{arg\,min}

%% file: def.tex
\usepackage[utf8]{inputenc}
\usepackage[T1]{fontenc}
\usepackage{booktabs}
\usepackage{bm}
\usepackage{bbm}
\usepackage{tcolorbox}
\usepackage{multicol,multirow}
\usepackage{natbib}
\usepackage{comment}
\usepackage{booktabs}

\usepackage{enumerate}   
\usepackage{amsmath}
\usepackage{tikz}

\usetikzlibrary{positioning,arrows}

\usepackage{wrapfig}

\usepackage{amsmath}
\usepackage{appendix}
\usepackage{amssymb}

\usepackage{amsthm}
\usepackage{hyperref}

\def\Holder{{H\"{o}lder}}

\newcommand{\pa}{\mathrm{\pa}}

\newcommand{\Risk}{\mathrm{UKL}}
\newcommand{\hRisk}{\widehat{\mathrm{UKL}}}

\newcommand{\annot}[2]{\underbrace{#1}_{\text{#2}}}
\def \Re {\mathbb{R}}
\def \Na {\mathbb{N}}

\def \r {r}
\def \Ltwo {L^2}

\def \loss {\ell}
\newcommand{\br}{f} % Bregman divergence potential

\def \lossOne {\loss_1}
\def \lossTwo {\loss_2}

\newcommand{\rClass}{\mathcal{H}}
\newcommand{\rClassBound}{B_r}
\newcommand{\rClassBoundMin}{b_r}
\newcommand{\rClassRangeTwo}{(\rClassBoundMin, \rClassBound)}

\def \Enu {\E_{\mathbb{P}}}
\def \Ede {\E_{\mathbb{Q}}}
\def \hEnu {\widehat{\E}_{\mathbb{P}}}
\def \hEde {\widehat{\E}_{\mathbb{Q}}}

\def \rstar {r^*}
\def \rbest {\bar{r}}
\def \hr {\hat{\r}}

\def \InSpace {\mathcal{X}}

\newcommand{\lOneR}{\lossOne(\r(X))}

\def \E {\mathbb{E}}

\newcommand{\lTwoR}{\lossTwo(\r(X))}
\newcommand{\lOne}[1]{\lossOne^{#1}}
\newcommand{\lTwo}[1]{\lossTwo^{#1}}
\newcommand{\lzeroNrm}[1]{\|#1\|_0}

\newcommand{\Order}[1]{O\left(#1\right)}

\newcommand{\Orderp}[1]{O_{\mathbb{P}}\left(#1\right)}

\newcommand{\Lmax}{\bar{L}}
\newcommand{\pmax}{\bar{p}}

\newcommand{\rClassM}{\rClass_M}
\newcommand{\bracketEntropy}[3]{H_B\left(#1, #2, #3\right)}

\newcommand{\rClassMSupNormBoundConst}{c_0}

\newcommand{\IndLP}{\mathrm{Ind}_{\Lmax,\pmax}}

\newcommand{\rClassLP}{\rClass(L, p, s, F)}
\newcommand{\elled}[1]{\ell \circ #1}
\newcommand{\ellLip}{\nu}
\newcommand{\vmax}[2]{\max\left\{#1, #2\right\}}